\documentclass[letterpaper]{article} 
\usepackage{aaai24}  
\usepackage{times}  
\usepackage{helvet}  
\usepackage{courier}  
\usepackage[hyphens]{url}  
\usepackage{graphicx} 
\usepackage{natbib}  
\usepackage{caption} 
\DeclareCaptionStyle{ruled}{labelfont=normalfont,labelsep=colon,strut=off} 
\usepackage{times}
\usepackage{epsfig}
\usepackage{graphicx}
\usepackage{amsmath}
\usepackage{amssymb}
\usepackage{amsthm}
\usepackage{booktabs}
\usepackage{enumitem}

\usepackage{algorithm}
\usepackage{adjustbox}
\usepackage{multirow}
\usepackage{float}
\usepackage{subcaption}
\usepackage[noend]{algpseudocode}
\usepackage[dvipsnames]{xcolor}
\usepackage{colortbl}
\usepackage{makecell}
\usepackage[capitalize]{cleveref}
\crefname{section}{Sec.}{Secs.}
\Crefname{section}{Section}{Sections}
\Crefname{table}{Table}{Tables}
\crefname{table}{Tab.}{Tabs.}

\newtheorem{theorem}{Theorem}[section]

\newcommand{\B}[1]{\textbf{#1}}
\newcommand{\I}[1]{\textit{#1}}

\newcommand{\s}[0]{\small}
\newcommand{\bt}[0]{\vartriangleright}
\newcommand{\lc}[1]{\multicolumn{1}{|l|}{#1}}

\title{Towards Improved Proxy-based Deep Metric Learning via Data-Augmented Domain Adaptation}

\author{
    Li Ren,
    Chen Chen,
    Liqiang Wang,
    Kien Hua
}
\affiliations{
    University of Central Florida\\
    \{Li.Ren, Chen.Chen, Liqiang.Wang, Kien.Hua\}@ucf.edu
}

\begin{document}

\maketitle

\begin{abstract}
	
Deep Metric Learning (DML) plays an important role in modern computer vision research, where we learn a distance metric for a set of image representations. Recent DML techniques utilize the \textit{proxy} to interact with the corresponding image samples in the embedding space. However, existing proxy-based DML methods focus on learning individual proxy-to-sample distance while the overall distribution of samples and proxies lacks attention. In this paper, we present a novel proxy-based DML framework that focuses on aligning the sample and proxy distributions to improve the efficiency of proxy-based DML losses. Specifically, we propose the \B{D}ata-\B{A}ugmented \B{D}omain \B{A}daptation (\B{DADA}) method to adapt the domain gap between the group of samples and proxies. To the best of our knowledge, we are the first to leverage domain adaptation to boost the performance of proxy-based DML. We show that our method can be easily plugged into existing proxy-based DML losses. Our experiments on benchmarks, including the popular CUB-200-2011, CARS196, Stanford Online Products, and In-Shop Clothes Retrieval, show that our learning algorithm significantly improves the existing proxy losses and achieves superior results compared to the existing methods. Our code is available at \url{https://github.com/Noahsark/DADA}
	
\end{abstract}

\begin{figure}[pt]
	\centering
	\includegraphics[width=0.45\textwidth,height=0.20\textheight]{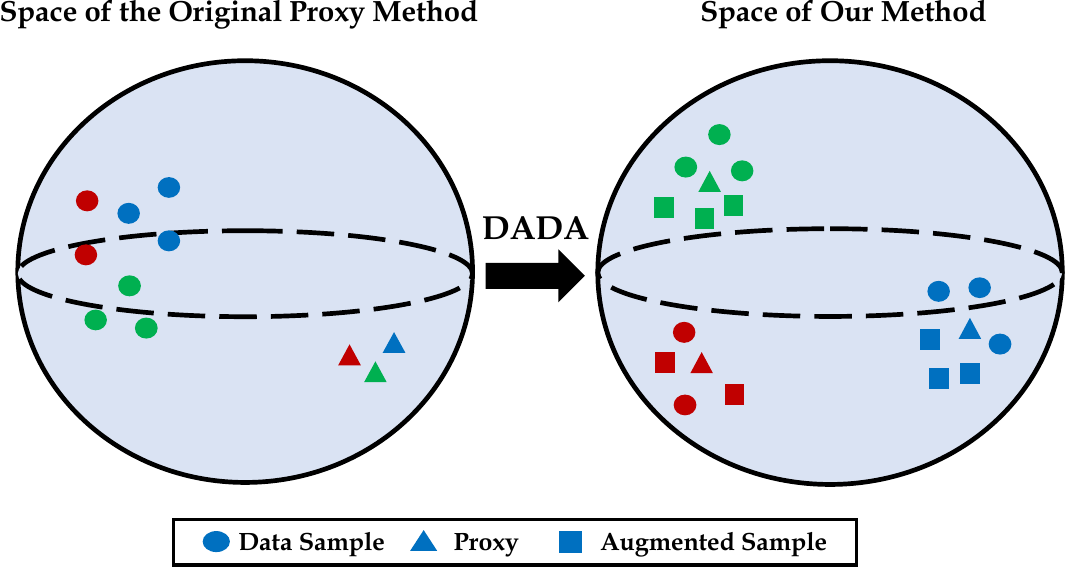}
	\caption{The intuition of our Data-Augmented Domain Adaptation (DADA). The classes are labeled with unique colors. The initial distribution gap between the data samples and corresponding proxies causes ambiguity for proxy-based deep metric learning. Our proposed method solves this problem by aligning the data samples and proxies, assuming they are from different data domains. We further augment the data to a dense manifold with mixed features to support this alignment.}
	\label{basic}
\end{figure}

\section{Introduction}
\label{sec:intro}

The fundamental task of Deep Metric Learning (DML) focuses on learning deep representation with a known similarity metric. DML is a crucial topic in computer vision since it has a wide range of applications, including image retrieval \cite{lee2008rank,yang2018retrieving,ren2021beyond}, person re-identification \cite{yi2014deep, wojke2018deep, dai2019batch}, and image localization \cite{lu2015localize,ge2020self}. Modern DML techniques utilize deep neural networks (DNN) to project image samples into a hidden space where similar data points are grouped within short distances while the dissimilar points are separated. The majority of DML approaches focus on optimizing the similarities between pairwise samples with various loss functions, ranging from contrastive losses \cite{hadsell2006dimensionality}, triplet losses \cite{schroff2015facenet} to cross-entropy losses \cite{boudiaf2020unifying}. With the increasing number of samples in the deep learning tasks, the basic pair or triplet losses face the difficulty of high computational complexity. Some approaches select informative samples by mining the hard or semi-hard samples \cite{wu2017sampling,katharopoulos2018not} while another group is devoted to comparing the sample clusters \cite{oh2017deep} or the statistics of the samples \cite{rippel2015metric}.

Unlike the pair-based DML methods, the proxy-based approaches try to learn a group of trainable vectors, named \I{proxy}, instead of sweeping all sample pairs within the mini-batch or cluster \cite{movshovitz2017no,kim2020proxy}. Thus, the proxies capture the semantic information about the classes and optimize the uninformative sample-sample comparison with the proxy-sample relations. Based on the efficient proxy-sample distance metrics, later works further select the most informative proxy \cite{zhu2020fewer} or assign each class with multiple proxies \cite{qian2019softtriple} to capture the intra-class structures. However, those existing proxy-based approaches simply guide the proxies by measuring their similarity with data samples where the learning process still faces a fundamental problem: \I{the colossal distribution gap between the proxies and the data samples}, since the proxies are initially sampled from a normal distribution that does not contain any semantic information. The distribution gap would slow the convergence speed and cause ambiguity and bias in the learning process. Initializing the proxy with representations of the data sample is one straightforward solution to this problem. However, the distribution of proxies still differs dramatically between the early and late training stages due to the poor quality of sample representations at the early stage. Additionally, it takes a significant amount of extra time and space to calculate the representations for every class in each iteration.

In this paper, we introduce a novel framework to solve these problems by aligning the distributions of the proxies and the data samples (as illustrated in Figure \ref{basic}). Specifically, we utilize Adversarial Domain Adaptation \cite{wang2018deep} techniques to minimize their distribution gap. To align those distributions, we propose a \I{domain-level discriminator}, which is a classifier to separate their domain properties. Note that the single domain discriminator would cause \I{mode collapse} \cite{goodfellow2020generative,che2016mode} where the majority of data points are constrained to a local area so that their discriminative information is lost. To endorse their discriminative information, we leverage one additional \I{category-level discriminator} to evaluate the consistency of their class properties. We show that with these discriminators, the adversarial training signal can efficiently align the distributions of the data samples and the proxies.

However, there are still two difficulties in learning the distribution of the proxy space: (1) the limited number and diversity of the proxies and (2) the large initial gap between the proxies and the data samples. The limited number of proxies causes difficulties for discriminators in capturing the inter-class manifold structure, and the large domain gap further hinders their learning efficiency. To overcome these challenges, we propose a novel data-augmented domain as a bridge where the data samples and the proxies are evenly mixed to conduct an intermediate domain. This domain contains rich mixing samples holding information and statistics from both sides. We also propose to create mixture samples within the same categories to increase the density of the manifold. We demonstrate the mechanisms of our method in Figure  \ref{fig:explain}. Our experiments show that the proposed method can easily plug into existing proxy-based losses to boost their performance dramatically. Our main contributions are three-fold:

\setlist{nolistsep}
\begin{itemize}[noitemsep,leftmargin=*] 
	
	\item We propose a novel adversarial learning framework to optimize the existing proxy-based DML by aligning the overall distributions of the data samples and the proxies at both domain and category levels. 
	
	\item We propose an additional data-augmented domain that contains mixup representations from both sides to further bridge the distribution gap. We show that our combined discriminators efficiently guide the proxies and the data samples to a hidden space under the same distribution, in which the proxy-based loss significantly increases its learning efficiency.
	
	\item Our experiments demonstrate the effectiveness of our adversarial adaptation method on the image data samples and the proxies. We show that our approach increases the performance of existing proxy-based DML loss by a large margin, and our best result outperforms the state-of-the-art methods on four popular benchmarks.

\end{itemize}
	\begin{figure*}[pt]
		\centering
		\includegraphics[width=0.90\textwidth,height=0.30\textheight]{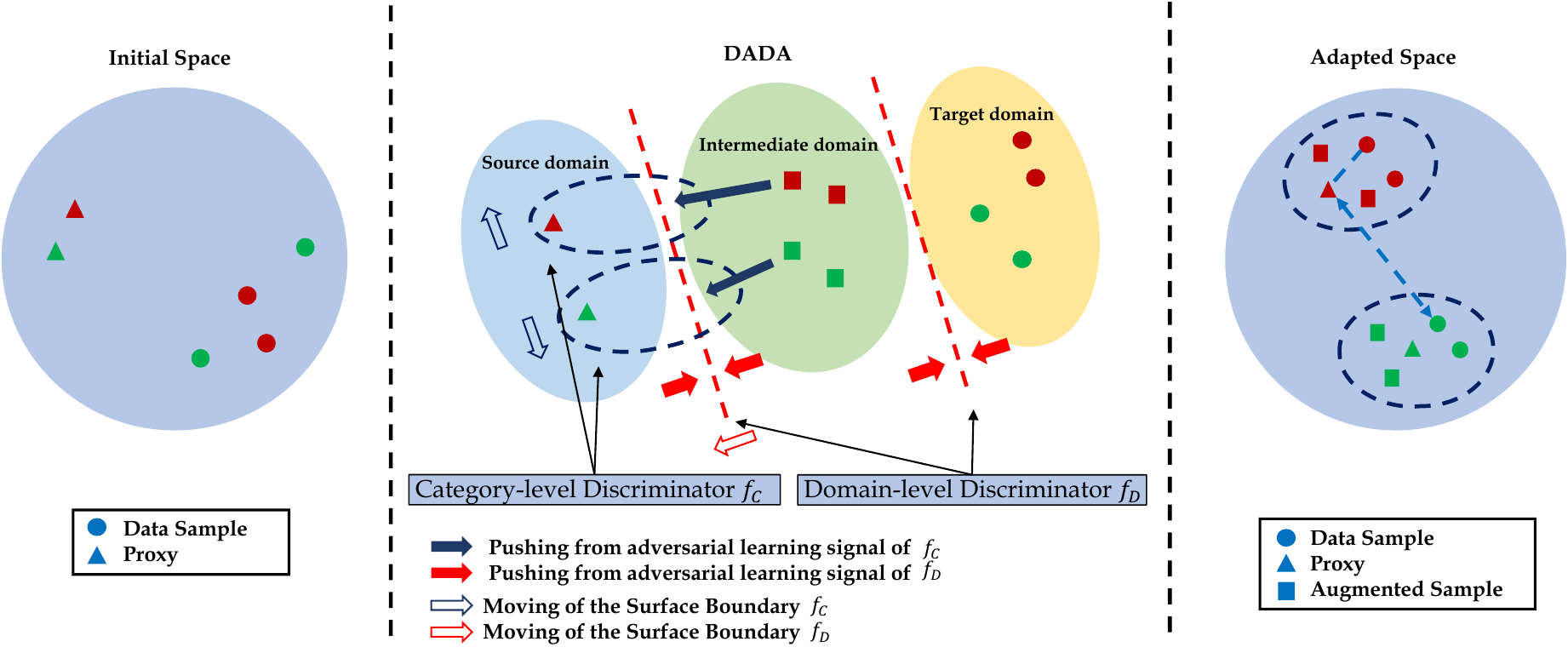}
		\caption{Demonstrate the mechanisms of our adversarial learning. Each class is labeled with a unique color. \B{Left}: Illustrate the Initial Space. \B{Mid}: Illustrate the training mechanisms and progress of our proposed method. \B{Right}: Illustrate the Adapted Space after training. The surface boundaries of the classifiers are trained to discriminate the domains with Domain-level Discriminators, and sample classes with Category-level Discriminators in the \I{discriminator training phase}. In the \I{generator training phase}, the samples and proxies are pushed to fool the Domain-level Discriminators from the adversarial learning signals while the class predictions from Category-level Discriminators are maintained.}
		\label{fig:explain}
	\end{figure*}

\section{Related Work}

\textbf{Pair-based DML.} Metric Learning in the computer vision area aims to learn a metric that measures the distance between a pair of image samples. Initially, the image samples inside a class and out of a class are regarded as \textit{positive} and \textit{negative} samples; and they are learned and projected to a low dimensional space \cite{hadsell2006dimensionality, oh2016deep}. The samples in different classes are paired and measured with the \textit{contrastive loss} \cite{chopra2005learning, hadsell2006dimensionality}. To further compare the ranking relation between pairs of samples, an additional sample is selected as an anchor to compare with both positive and negative samples with the \textit{triplet loss} \cite{weinberger2009distance,wang2014learning,cheng2016person,hermans2017defense} where the positive sample is ensured to be closer than the negative samples. Based on the triple loss, Sohn et al.\shortcite{sohn2016improved} propose SoftMax cross-entropy to compare the group of pairs to improve pair sampling. 

The computational cost of these pair-based works is always high due to the workload of comparing each sample with all other samples within a given batch. Additionally, these methods reveal sensitivity to the size of the batch, where their performance may significantly drop if the size is too small.

\textbf{Proxy-based DML.} To further accelerate the sampling and clustering process, Movshovitz et al. \shortcite{movshovitz2017no} leverage the \textit{proxy}, a group of learnable representations, to compare data samples via the Neighbourhood component analysis (NCA) loss \cite{roweis2004neighbourhood}. The motivation is to set image samples as anchors to compare with proxies of different classes instead of corresponding samples to reduce sampling times. Teh et al. \shortcite{teh2020proxynca++} further improve the ProxyNCA by scaling the gradient of proxies. Zhu et al. \shortcite{zhu2020fewer} propose to sample the most informative negative proxies to improve the performance, while Kim et al. \shortcite{kim2020proxy} set the proxies as anchors instead of the samples to learn the inter-class structure. Yang et al. \shortcite{yang2022hierarchical} develop hierarchical-based proxy loss to boost learning efficiency. Roth et al. \shortcite{roth2022non} regular the distribution of samples around the proxies following a non-isotropic distribution. In contrast to these methods that compare the single sample-proxy pair, our method further refines the manifold structure by aligning the whole distributions between proxies and image samples via a novel adversarial domain adaptation framework.

\B{Domain Adaptation and Adversarial Learning.} Domain Adaptation initially aims to solve the lack of labeled data where the learned feature is domain-invariant so that classifiers can be easily shifted to the new data distribution. The basic idea is to match the feature distributions to decrease their \textit{domain shift} between the source and target datasets \cite{quinonero2008covariate,torralba2011unbiased}. One important branch of domain adaptation is Adversarial Learning \cite{goodfellow2020generative,hassanpour2022survey}, where two or more models take part in the min-max game to generate domain-invariant features. 

Ganin et al. \shortcite{ganin2015unsupervised} first generate domain invariant features with adversarial training on neural networks. Tzeng et al. \shortcite{tzeng2017adversarial} improve the discriminator that does not share the weight with the feature generator. Pei et al. \shortcite{pei2018multi} utilize multiple discriminators assigned for each class to improve performance. Saito et al. \shortcite{saito2018maximum} minimize the prediction discrepancy of two discriminators on the target domain, while Lee et al. \shortcite{lee2019sliced} improve the method to compare their sliced Wasserstein distance instead. The primary application of adversarial learning is to produce synthetic textual or image data \cite{kingma2013auto,radford2015unsupervised,isola2017image}. Ren et al. \shortcite{ren2018improved, ren2019improving} also applied this technique to enhance the quality of image captioning.

Recent studies have investigated the application of domain adaptation in image or textual retrieval tasks \cite{laradji2020m,pinheiro2018unsupervised}. Wang et al. \shortcite{wang2017adversarial} employ domain adaptation to align image and textural data using a single discriminator, whereas Ren et al. \shortcite{ren2021beyond} utilize multiple discriminators to get improved performance. In contrast to previous efforts, we propose aligning the distributions of data representations and proxies within the same image modality.

\section{Proposed Method}
\label{sec:method}

We propose a new framework to close the gap between the distributions of the data samples and the proxies for proxy-based DML losses that are already in place. We utilize the adversarial domain adaptation technique to transfer data samples and proxies to domain invariant feature space. To overcome the limitation of the number of proxies, we also conduct a novel strategy to augment data as a bridge between the samples and proxies, which demonstrates a smooth learning process.

\subsection{Preliminary}
\label{sec:preliminary}

\B{Deep Metric Learning (DML)} Consider a set of data samples $\mathcal{S}=\{I_i, y_i\}_{i=1}^N$ with raw images $I_i$ and its corresponding class label $y_i \in \{1,\dots,C\}$; we learn a projection function $f_G:\mathcal{S} \stackrel{f}{\rightarrow} \mathcal{X}$, which project the input data samples to a hidden embedding space (or metric space) $\mathcal{X}$. We define the projected features set as $X = \{x_i\in \mathbb{R}^d\}_{i=1}^N$. The primal goal of Deep Metric Learning (DML) is to refine the projector function $f_G(\cdot)$, which is usually constructed with \textit{convolutional deep neural networks} (CNN) as the backbone, to generate the projected features that can be easily measured with defined distance metric $d(x_i, x_j)$ based on the semantic similarity between sample $I_i$ and $I_j$. Here we adopt the distance metric $d(\cdot)$ as the \I{cosine similarity}. Before delivering features to any loss, we use L2 normalization to eliminate the effect of differing magnitudes.

\B{Proxy-based DML} To boost the learning efficiency, a group of DML methods pre-define a set of learnable representations $P=\{p_i\in \mathbb{R}^d\}_{i=1}^C$, named \textit{proxy}, to represent subsets or categories of data samples. Typically there is one proxy for each class so that the number of proxies is the same as the number of classes $C$. The proxies are also optimized with other network parameters. The first proxy-based method, Proxy-NCA \cite{movshovitz2017no}, or its improved version Proxy-NCA++ \cite{teh2020proxynca++}, utilizes the Neighborhood Component Analysis (NCA) \cite{goldberger2004neighbourhood} loss to conduct this optimization. The later loss Proxy-Anchor (PA) \cite{kim2020proxy} inversely sets the proxy as the anchor and measures all proxies for each minibatch of samples. The PA loss $\mathcal{L}_{proxy}(X, P)$ can be presented as

\begin{equation}
	\label{loss_pa}
	\scalebox{0.85}{
	$\begin{aligned}
		\mathcal{L}_{proxy}(X, P) &=\frac{1}{\left|P^{+}\right|} \sum_{p \in P^{+}} \log \left(1+\sum_{x\in 
			X_{p}^{+}} e^{-\tau d(x, p)+\delta}\right) \\ &+\frac{1}{|P|} \sum_{p \in P} \log \left(1+\sum_{x \in X_{p}^{-}} e^{\tau d(x, p)+\delta} \right)		
	\end{aligned}$
	}
\end{equation}
where $X_{p}^{+}$ denotes the set of positive samples for a proxy $p$; $X_{p}^{-}$ is its complement set; $\tau$ is the scale factor; and $\delta$ is the margin. Since the PA updates all proxies for each mini-batch, the model has higher learning efficiency in capturing the structure of samples beyond the mini-batches. We propose these two fundamental proxy-based losses (PNCA++ and PA) that achieve competitive results as our baselines.

\subsection{Domain Data Augmentation}
\label{sec:domain_mixup}
To reduce the distribution gap between the data samples and the proxies, we transform the proxy-based DML into a domain adaptation problem. We regard the data samples as data points in the source domain, while the initialed proxies are data points in the target domain. We noticed that the number of proxies in the target domain is especially limited compared to the data samples because the basic proxy method only assigns a single proxy for each class. The unbalanced samples and proxies would cause learning biases in modeling the distributions. Also, the proxies are initialized from a normal distribution and do not contain any related semantic information, which also causes difficulty in aligning their distribution to the data sample domain.

To overcome these difficulties, we propose a novel data augmentation strategy to create an intermediate domain to balance the amount of data points for domain adaptation. Specifically, we interpolate the space with mixed features from data set $X$ and proxy set $P$. For each data sample $x_i$ and its corresponding proxy $p_i$, we create a data feature $\hat{d}$:
\begin{equation}
	\hat{d}_i = \{\lambda x_i + (1 - \lambda)p_i \},
\end{equation}
where $\lambda \sim Beta(\alpha, \beta)$ is the linear interpolation coefficient that sampled from \I{beta} distribution with $\alpha > 0$ and $\beta > 0$ that decide its probability density function. The new data contains semantic information between the data sample and proxies and shares their distribution statistics. Therefore pushing $\hat{d}$ is equal to pushing both samples $x$ and proxies $p$, and their distribution is closer to the data sample domain than the original proxies.

In addition, we further propose extending the number of training instances by augmenting the data-proxy pairs within the same class. For each pair of samples $(x_i, x_j)$ and their corresponding augmented data $(\hat{d}_i, \hat{d}_j)$ inside the mini-batch, we propose the following mixing:

\begin{equation}
	\begin{aligned}
		\tilde{x}_i = \{ \mu_1 x_i + (1 - \mu_1)x_j\} \\
		\tilde{d}_i = \{ \mu_2 \hat{d}_i + (1 - \mu_2)\hat{d}_j \},
	\end{aligned}
\end{equation}
where $\mu_1, \mu_2$ are also sampled from $Beta$ distribution. Then we mix the new samples $\tilde{x}_i$ and $\tilde{d}_i$ inside the mini-batch to ensure the number of data samples with the same label $n \geq 2$. Combined with the original mini-batch, the augmented data sample set and the augmented proxy set are noted as $\tilde{X} = X \cup \{\tilde{x}_1,\tilde{x}_2\cdots\}$ and $\tilde{D} = \{\hat{d}_1, \hat{d}_2 \cdots\} \cup \{\tilde{d}_1,\tilde{d}_2 \cdots\}$, and the size of mini-batch is also extended accordingly. We then normalize the composed features in $\tilde{X}$ and $\tilde{D}$ with L2 normalization to constrain them on a unit hypersphere embedding space where the magnitude is fixed to 1.

\subsection{Domain-level Discriminator}
\label{sec:domain_level}
Based on the augmented data, our goal is to refine the set $\tilde{X}$, $\tilde{D}$ and the original proxy set $P$ to \textit{domain invariant} representations that share the same distribution to help the proxy-based losses. We follow the principle idea of adversarial domain adaptation \cite{ganin2016domain} to estimate the domain divergence by learning a domain-level discriminator. Specifically, we learn a classifier $f_D(\cdot)$ that minimizes the risk of \textit{domain prediction} (to predict if the data comes from a unique domain) between the set $\tilde{X}$, $\tilde{D}$ and $P$.

Generally, we would label the data from a specific domain with the one-hot label as the prediction target. Since we have three different domains including the augmented data domain and our labeling space is symmetric, we would simply assume the features $\tilde{x} \in \tilde{X}$ are labeled as $y_0 = \overline{001}$ while $\tilde{d} \in \tilde{D}$ are labeled as $y_1 = \overline{010}$, and the initial proxies $P$ are labeled as $y_2 = \overline{100}$ for convenience. Specifically, we estimate the domain classifier $f_D(\cdot)$ as an MLP with a single hidden layer and a ReLU function. The hidden layer is then projected to a 3-dimensional head as the logits prediction of the domains. To optimize the $f_D(\cdot)$ with a low prediction risk on the labeling space, we conduct the cross-entropy objective $\mathcal{L}_{adv}$ as follows

\begin{equation}
\scalebox{0.90}{
	$\begin{aligned} 
		&\mathcal{L}_{adv}(\tilde{X}, \tilde{D}, P) = \sum_i^{\tilde{N}} \mathcal{L}_{ce}(f_D(\tilde{x}_i), y_0) 			\\
		& + \sum_i^{\tilde{N}} \mathcal{L}_{ce}(f_D(\tilde{d}_i), y_1) + \sum_i^{C} \mathcal{L}_{ce}(f_D(p_i), y_2)
	\end{aligned}$
}
\label{eq:adv}
\end{equation}

where $\mathcal{L}_{ce}$ is the cross entropy loss and $\tilde{N}$ is the total number of samples after the data augmentation. The parameters of the classifier $f_D(\cdot)$ are optimized to minimize the adversarial loss $\mathcal{L}_{adv}$ in training. Recall that the feature $x$ is generated from the projection function $f_G(\cdot)$. Thus, the parameters of generator $f_G(\cdot)$ are optimized to fool the discriminator $f_D(\cdot)$ in the opposite direction. Since $\tilde{D}$ in the target domain contains features that mixed from $X$ and proxies $P$, optimizing the $\tilde{D}$ equals optimizing the generator $f_G(\cdot)$ in the source domain while updating the original proxies $P$. Thus, the adversarial learning signal of $\mathcal{L}_{adv}$ would help both generator $f_G(\cdot)$ and the original proxies $P$ to maintain the domain invariant representations to fool the classifier.

\subsection{Category-level Discriminator}
\label{sec:class_level}

One drawback of the domain-level discriminator described above is that the discriminate information, especially the inter-class correlation, is ignored in the optimization process. Losing the discriminative information will cause all data points to be concentrated on a local area or a surface, which would cause inter-class ambiguity and confuse the metric learning losses. To solve this problem, we further propose a category-level discriminator that learns to predict the class of data samples and compare the discrepancy of predictions between the data samples and mixture proxies.

Specifically, we optimize a classifier $f_C(\cdot)$ with the feature generator $f_G(\cdot)$ to predict the category label $Y=\{y_0,y_1,\dots\}$ from mixture data samples $\tilde{X}$ with the classification loss $\mathcal{L}_{cls}(\tilde{X}, Y)$ as
\begin{equation}
	\mathcal{L}_{cls}(\tilde{X}, Y) = \frac{1}{\tilde{N}}\sum_i^{\tilde{N}} \mathcal{L}_{ce}(f_C(\tilde{x_i}), y_i).
	\label{eq:cls}
\end{equation}
The cross-entropy loss $\mathcal{L}_{ce}$ would provide a supervised learning signal to $f_G(\cdot)$ to maintain the discriminative information during the DML training process.

We note that the data samples that share the distributions would also share the labeling space with the target proxy domain. To further align the distributions, we propose to constrain the samples from the source domain and augmented data from the target domain to have a low discrepancy of predictions from our category classifier $f_C(\cdot)$. Thus, one additional goal of $f_C(\cdot)$ is to learn the maximized discrepancy of the category prediction between the data samples $\tilde{X}$ and mixture proxies $\tilde{D}$ while the $\tilde{D}$ are later optimized to minimize this discrepancy.

To measure the discrepancy of the category probabilities, we empirically adopt the discrepancy introduced in \cite{chen2022reusing} that utilizes the \I{Nuclear-norm Wasserstein Distance (NWD)}. The NWD is demonstrated to be the upper bound of the \textit{Frobenius-norm}, which estimates the correlations of the predictions \cite{cui2020towards}. Thus, we compare the NWD between the logistic predictions of $f_C(\cdot)$ from the augmented samples $\tilde{X}$ and data $\tilde{D}$. The loss $\mathcal{L}_d (\tilde{X}, \tilde{D})$, which measures the NWD can be described as,
\begin{equation}
	\mathcal{L}_d (\tilde{X}, \tilde{D}) = \frac{1}{\tilde{N}}(\sum_i^{\tilde{N}} ||f_C(\tilde{X})||_{\star} - \sum_i^N ||f_C(\tilde{D})||_{\star}),
	\label{eq:discrepancy}
\end{equation}
where $||x||_{\star} = \sum \sigma(x)$ denotes the \textit{nuclear-norm} of $x$, which is defined as the sum of its singular values.

\subsection{The Combined Loss and Training Progress}
\label{sec:theory}

We adopt the paradigm of adversarial learning to alternatively update the gradient of our feature generator $f_G(\cdot)$ and the discriminators $f_D(\cdot)$ and $f_C(\cdot)$ discussed above. To this end, we train our combined loss by playing the \I{min-max game} as follows,
\vspace{2pt}
\begin{equation}
	\min_{f_G, f_C} \{ \eta (\mathcal{L}_{cls} + \max_{f_C} \mathcal{L}_d)\} + (1-\eta) \min_{f_D} \max_{f_G}\mathcal{L}_{adv},
\end{equation}
\vspace{2pt}
where $\eta$ is the pre-defined hyperparameter that balances the contribution between the domain-level and category-level discriminators. Empirically, we do not set another weight between classification loss $\mathcal{L}_{cls}$ and discrepancy loss $\mathcal{L}_d$. We also need the original proxy-based loss $\mathcal{L}_{proxy}$ in Eq. \ref{loss_pa} to do the basic DML of the sample-proxy pair in training. Note that the augmented data set $\tilde{D}$ is only for domain adaptation progress; the original $\mathcal{L}_{proxy}$ only operates $\tilde{X}$ and original proxies $P$. Thus, our combined training progress can be described as the following two sub-processes:

\begin{flushleft}
\scalebox{0.85}{
\begin{minipage}{1.15\linewidth}
\begin{align}
&(\theta_{f_D}, \theta_{f_C}) = \arg\min_{f_D, f_C} \{ \eta (\mathcal{L}_{cls} - \mathcal{L}_d) +  (1-\eta) \mathcal{L}_{adv} \}, \label{eq:update_1} \\
&(\theta_{f_G}, P) = \arg\min_{f_G, P, \tilde{D}} \{ \eta (\mathcal{L}_{cls} + \mathcal{L}_d) - (1-\eta) \mathcal{L}_{adv} + \gamma \mathcal{L}_{proxy} \}, \label{eq:update_2}
\end{align}
\end{minipage}
}
\end{flushleft}

where parameters $\theta_{f_D}$ and $\theta_{f_C}$ are updated in first phase and $\theta_{f_G}$ and the proxies $P$ are updated with $\tilde{D}$ in the second phase. Even if gradient reversal layers are accepted for achieving adversarial training in earlier domain adaptation works, we empirically conclude that a separate training phase would be more feasible for us in our search for stable training parameters. The full training progress can be referred to in Algorithm \ref{training}.


\begin{table*}[h]
	\small
	\centering
	\begin{adjustbox}{max width=0.90\textwidth}
		\begin{tabular}{|c|c|c|c|c|c|c|c|c|c|c|c|} \hline 
			\multirow{ 2}{*}{Method} & \multirow{ 2}{*}{Reference} &Settings & \multicolumn{3}{|c|}{CUB-200} & 	\multicolumn{3}{|c|}{CARS-196}  & \multicolumn{3}{|c|}{SOP}  \\ 
			\cline{3-12}
			
			&				&Arch/Dim& R@1  & R@2 	& R@4 			& R@1   & R@2    & R@4  	&  R@1	& R@10 &  R@100  		\\  \hline
			
			PNCA \cite{movshovitz2017no}	&\s{CVPR17'}	&BN/512	&49.2	&61.9	&67.9				&73.2	&82.4	 &86.4				&73.7	&--		&--		\\ 
			ProxyGML$\dagger$\cite{zhu2020fewer}&\s{NeurIPS20'}&BN/512&66.6	&77.6	&86.4	 			&85.5	&91.8	 &95.3			&78.0	&90.6	&96.2		\\ 
			DiVA\cite{milbich2020diva}&\s{ECCV20'}			&R50/512&69.2	&79.3	&--					&87.6	&92.9	 &--	  						&79.6	&91.2	&--		\\ 
			S2SD\cite{roth2021simultaneous}&\s{ICML21'}		&R50/512&70.1	&79.7				&71.6		&89.5	&93.9			&72.9	&80.0	&91.4	&--			\\	
			DCML-Proxy$\dagger$\cite{zheng2021dcml}&\s{CVPR21'}		&R50/512			&65.2	&76.4	&84.8						&81.2	 &89.8	&94.6		&--		&--		&--		\\ 
			DCML-MDW\cite{zheng2021dcml}&\s{CVPR21'}		&R50/512		&68.4	&77.9	& 86.1		&85.2	&91.8	 &96.0		&79.8	&90.8	&95.8		\\ 
			DRML\cite{zheng2021drml}&\s{ICCV21'}			&BN/512	&68.7	&78.6	& 86.3						&86.9	&92.1	 &95.2		&71.5	&85.2	&93.0		\\ 
			
			PA+AVSL$\dagger$\cite{zhang2022attributable} 	&\s{CVPR22'}	&R50/512	&71.9 & 81.7 & 88.1 	&91.5&95.0&97.0	&79.6	&91.4	&96.4		\\ 
			PA+NIR$\dagger$\cite{roth2022non} &\s{CVPR22'}	&R50/512		&69.1	& 79.6 	& -- 		&87.7 	&92.5 	& --			&80.7	&91.5	&--		\\ 
			HIST\cite{lim2022hypergraph} &\s{CVPR22'}		&R50/512&71.4 	& 81.1	& 88.1 		&89.6   &93.9 	 & 96.4 		&\B{81.4}	&92.0	&\B{96.7} \\ 
			DAS\cite{liu2022densely} &\s{ECCV22'}			&R50/512&69.2 	& 79.3 	& 87.0 			&87.8 	&93.2   & 96.0		&80.6	&91.8	&\B{96.7}	\\ 
			MS+CRT\cite{kan2022coded} &\s{NeurIPS22'}		&R50/512	&64.2 	& 75.5 	& 84.1 	&83.3 	&89.8    & 93.9 	&79.0	&91.1	&96.5		\\ 
			\hline \hline
			
			$\bt$PNCA++\cite{teh2020proxynca++}&\s{ECCV20'}	&R50/512		&69.0	&79.8	&87.3				&86.5	 &92.5	&95.7			&80.7	&92.0	&\B{96.7}	\\ \hline		
			PNCA+\B{DADA}(R50) & Ours 						&R50/512&71.4	&81.1	&87.6			&90.5	 &93.4	&96.8			&81.2	&91.8	&96.5			\\ \hline
			
			$\bt$PA\cite{kim2020proxy}&\s{CVPR20'}	&BN/512&68.4	&79.2	&86.8		&86.1	 &91.7	&95.0			&79.1	 &90.8	 &96.2	\\ \hline

			PA+\B{DADA}(BN)	 & Ours							&BN/512&69.8	&80.4	&87.1			&89.4	 &92.1	 &96.2		&79.6	 &91.0	 &96.3	\\ \hline
			
			$\bt$PA (R50) $\ast$ \cite{kim2020proxy}&\s{CVPR20'}		&R50/512&69.7	&80.0	&87.0		&87.7	 &92.9	 &95.8		 &80.0	&91.7	 &96.6	\\ \hline
			PA+\B{DADA}(R50) & Ours 					&R50/512&\B{72.9}&\B{81.9}&\B{88.3}		&\B{92.1}&\B{95.2}&\B{97.1}	&{81.0}&\B{92.1}& {96.2}		\\ \hline
		\end{tabular}
	\end{adjustbox}
	\caption{Comparison with the state-of-the-art litterateurs on CUB200-2011 \cite{wah2011caltech}, CARS196 \cite{krause20133d}, Stanford Online Products (SOP) \cite{oh2016deep}. The works are sorted by their published date. The second column shows the same architecture of the backbone and feature dimension we selected to compare with our proposed method. \B{R50} represents the ResNet50 and \B{BN} for InceptionBN and \B{GN} for GoogleNet backbones. $\dagger$ denotes the methods based on proxy-based DML, and $\bt$ labels the works on which our method is based. We adopt the experimental results of \B{PA(R50)} from the third papers \cite{lim2022hypergraph}. The \B{Bold} represents the best score.}
	\label{table:sota_1}
\end{table*}


	
	\begin{algorithm}
	\small
		\caption{Data-Augmented Domain Adaptation (DADA) for Proxy-based Deep Metric Learning}
		\label{training}
		\resizebox{0.95\linewidth}{!}{%
		\begin{minipage}{\linewidth}
			\begin{algorithmic}[1]
				\State \B{Input: } Training Set $\mathcal{S}=\{I_i, y_i\}_{i=1}^N$
				
				\State \B{Initialization:} $\theta_{f_G}$, $\theta_{f_C}$, $\theta_{f_D}$, and proxies $P$
				\While {\I{stop criteria is not satisfied}}
				
				\State Obtain a batch $\{I_i, y_i\}_{i=1}^n$ from $S$
				\State Select proxies $P = \{p_i\}_{i=1}^{n}$ according the labels $Y$
				\State Embedding features $X=\{x_i\}_{i=1}^{n}\gets f_G(I)$
				
				\State /* \I{Prepare the mixture intermediate data domain} */
				\State Sample $\lambda \sim Beta(\alpha, \beta)$
				\State Sample $\mu_1, \mu_2 \sim Beta(1.0, 1.0)$
				\State Compose $\hat{P} \gets \{\lambda X + (1 - \lambda)P\}$ 
				\State Compose $\tilde{X} \gets X \cup \{\mu_1 x_i + (1 - \mu_1) x_j \}$
				\State Compose $\tilde{P} \gets \hat{P} \cup \{\mu_2 \hat{p}_i + (1 - \mu_2) \hat{p}_j \}$
				\State L2 Normalize $\tilde{X}$, $\tilde{P}$ and $P$
				
				\State /* \I{Discriminator Training Phase} begin */
				\For {$k$ steps}
				\State Cal $ \Delta\theta_{f_D},\Delta\theta_{f_C} \gets \eta\frac{\partial(\mathcal{L}_{cls}(\tilde{X}, Y) - \mathcal{L}_d(\tilde{X}, \tilde{P}))}{\Delta\theta_{f_D},\Delta\theta_{f_C}}$
				\State Cal $\Delta\theta_{f_D},\Delta\theta_{f_C} \gets (1 - \eta)
				\frac{\partial \mathcal{L}_{adv}(\tilde{X}, \tilde{P}, P)}{\Delta\theta_{f_D},\Delta\theta_{f_C}}$
				\State Update $\theta_{f_D}, \theta_{f_C} \gets Adam \{\Delta\theta_{f_D}, \Delta\theta_{f_C} \}$
				\EndFor
				
				\State/* \I{Generator Training Phase} begin */
				
				\State Cal $\Delta\theta_{f_G},\Delta P \gets \eta \frac{\partial(\mathcal{L}_{cls}(\tilde{X}, Y)
					+ \mathcal{L}_d(\tilde{X}, \tilde{P}))}{\Delta\theta_{f_G}, \Delta P}$
				\State Cal $\Delta\theta_{f_G},\Delta P \gets -(1 - \eta)\frac{\partial \mathcal{L}_{adv}(\tilde{X}, \tilde{P}, P)}
				{\Delta\theta_{f_G},\Delta P}$
				\State Cal $\Delta\theta_{f_G},\Delta P \gets \gamma \frac{\partial \mathcal{L}_{proxy}(\tilde{X}, P)}{\Delta\theta_{f_G},\Delta P}$
				
				\State Update $\theta_{f_G}, P \gets Adam \{\Delta\theta_{f_G},\Delta P \}$
				\EndWhile
			\end{algorithmic} 
			\end{minipage}
			}
\end{algorithm}


\section{Experiments}
\label{sec:experiments}
We present our performance study and discuss the experimental results in this section.  
\subsection{Datasets and Metrics}

We use the standard benchmarks CUB-200-2011 (\B{CUB200}) \cite{wah2011caltech} with 11,788 bird images and 200 classes, and \B{CARS196} \cite{krause20133d} that contains 16,185 car images and 196 classes. We also evaluate our method on larger Stanford Online Products (\B{SOP}) \cite{oh2016deep} benchmark that includes 120,053 images with 22,634 product classes, and In-shop Clothes Retrieval (\B{In-Shop}) \cite{liu2016deepfashion} dataset with 25,882 images and 7982 classes. We follow the data split that is consistent with the standard settings of existing DML works \cite{teh2020proxynca++,kim2020proxy,mix2021,zheng2021drml,roth2022non,lim2022hypergraph,zhang2022attributable}. We adopt the \B{Recall@K} (K=1,2,4 in CUB200 and CARS196, K=1,10,100 in SOP, and K=1,10,20,30 in In-Shop) proposed in existing works to evaluate the accuracy of ranking. We also evaluate it with Mean Average Precision at R (\B{MAP@R}) which is based on the idea of MAP and R-precision, which is a more informative DML metric \cite{musgrave2020metric}.

\subsection{Implementation Details}

We train our model in a machine that contains a single RTX3090 GPU with 24GB memory. The Implementation is based on the existing RDML \cite{roth2020revisiting}

\B{Backbones and Preprocessing.} In this paper, we propose two basic backbones to evaluate our learning algorithm: the ResNet50\cite{he2016deep} and the InceptionBN \cite{ioffe2015batch}. They are pre-trained on ImageNet1K\cite{deng2009imagenet} and are widely used in DML works for performance evaluation, where we resize the image to $224 \times 224$, do random resized cropping, and random horizontal flipping. In the test phase, the images are first resized to $256 \times 256$, then cropped back to $224 \times 224$. A linear head embeds the feature from the second last layer of the backbones to a 512-dimension hidden space. We follow the standard pre-processing introduced in other deep metric learning works \cite{mix2021,zheng2021drml,roth2022non,lim2022hypergraph,zhang2022attributable}. We also adopt global max and average pooling with layer normalization on CNN backbones suggested by Teh et al. \cite{teh2020proxynca++} to further improve the generalization of features.


\begin{table}
	\small
	\begin{adjustbox}{max width=0.45\textwidth}
		\begin{tabular}{|c|c|c|c|c|c|} \hline 
			\multicolumn{6}{|c|}{In-Shop Clothes Retrieval (In-Shop)} \\ \hline
			Methods					 & Arch/Dim	&R@1		&R@10		&R@20  		&R@30 	\\ \hline
			MS \shortcite{wang2019multi}		&BN/512		&89.7		&97.9		&98.5		&98.8 	\\ \hline
			SHM \shortcite{suh2019stochastic}&BN/512		&90.7		&97.8		&98.5		&98.8 	\\ \hline		
			SCT \shortcite{xuan2020hard}		&R50/512	&90.0		&97.5		&98.1		&--		\\ \hline
			XBM \shortcite{wang2020cross}	&BN/512		&89.9		&97.6		&98.4		&98.6	\\ \hline
			IBC	\shortcite{ibc2021}			&R50/512	&92.8	&\B{98.5}	&\B{99.1}	&\B{99.2}	\\	\hline
			PA$\dagger$ \shortcite{kim2020proxy}		&BN/512		&90.4		&98.1		&98.8	&99.0	\\ \hline
			PA+Mix$\dagger$ \shortcite{mix2021}		&R50/512	&91.9		&98.2		&98.8	&--		\\ \hline
			\hline
			PNCA++$\dagger$\shortcite{teh2020proxynca++}&R50/512	&90.4		&98.1		&98.8		&99.0\\	\hline 
			PNCA + \B{DADA} (ours)		&R50/512	&91.7		&98.2		&98.6		&98.8		\\ \hline
			PA + \B{DADA} (ours)		&R50/512	&\B{93.0}	&\B{98.5}	& 98.9 & 99.1 \\ \hline
		\end{tabular}
	\end{adjustbox}
	\caption{Compare with the existing state-of-the-art DML works on the In-Shop \cite{liu2016deepfashion} dataset. The \B{Bold} represents the best score.}
	\label{table:sota_Inshop}
\end{table}



\begin{table}
	\small
	\centering
	\begin{adjustbox}{max width=0.48\textwidth}
		\begin{tabular}{|c|c|c|c|c|} \hline 
			& \multicolumn{2}{|c|}{ProxyAnchor}  & \multicolumn{2}{|c|}{ProxyNCA++} 	\\ \hline
			Settings							&R@1					& MAP@R 		&R@1	&MAP@R  		\\ \hline
			
			\lc{Baseline}																&69.1											& 26.5 								& 68.4						&	25.8							\\ \hline
			\lc{+Aug}																	& 69.3 (+0.2)			& 26.5 (+0.0)				& 68.5 (+0.1)				& 25.9 (+0.1)	\\ \hline
			\lc{$+\mathcal{L}_{adv}$}									&70.2 (+1.1)		 		& 27.3 (+0.8)				& 69.2 (+0.8) 				& 26.4 (+0.6)	\\ \hline
			\lc{ $+\mathcal{L}_{adv}+$Aug}						&70.9 (+1.8)		 		& 27.8 (+1.3)				& 69.8 (+1.4) 				& 26.8 (+1.0)	\\ \hline
			\lc{$+\mathcal{L}_{cls}$}										&69.3 (+0.2)		 		&	27.0 (+0.5)				&	68.9 (+0.5)				& 26.2 (+0.4)	\\ \hline
			\lc{$+\mathcal{L}_{cls}+\mathcal{L}_d$}		&69.9 (+0.8)	 			& 27.4 (+0.9)				& 69.5 (+1.1)				& 26.6 (+0.8)	 \\ \hline
			
			\lc{$+\mathcal{L}_{cls}+\mathcal{L}_d+$Aug}		&70.4 (+1.3)		& 27.8 (+1.3)		& 69.4 (+1.0)						&	26.7 (+0.9)	\\ \hline
			\lc{$+\mathcal{L}_{adv}+\mathcal{L}_{cls}$}			& 71.4 (+2.3)	 & 28.2 (+1.7)		&	 69.3 (+0.9)						&	27.1 (+1.3)	\\ \hline
			
			\lc{$+\mathcal{L}_{adv}+\mathcal{L}_{cls}+\mathcal{L}_1$} &71.6 (+2.5)	& 28.2 (+1.7)	& 69.4 (+1.0)	& 27.0 (+1.2)	\\ \hline
			
			\lc{$+\mathcal{L}_{adv}+\mathcal{L}_{cls}+\mathcal{L}_d$} &72.0 (+2.9)	& 28.9 (+2.4)	& 69.9 (+1.5)	& 27.7 (+1.9)	\\ \hline
			\makecell {$+\mathcal{L}_{adv}+\mathcal{L}_{cls}+\mathcal{L}_d $ \\
				+ Aug (\B{ours})}	&\B{72.9 (+3.8)}  &	\B{29.9(+3.4)}	& \B{70.2(+1.8)}		& \B{28.0 (+2.2)}	\\ \hline
		\end{tabular}
	\end{adjustbox}
	\caption{Study the contribution of each component of our method and loss function on CUB200. We reproduce the result of ProxyAnchor, which has a batch size of 90, and ProxyNCA++, which has a batch size of 32, as the baseline of our method. \B{Aug} represents the alignment of the augmented data and samples we introduced in Sec \ref{sec:domain_mixup}. We denote the difference in \I{percentage point} ($pp$) compared with our baseline in the bracket.}
	\label{table:ablation}
\end{table}


\B{Training Details.} Our optimization is done using Adam ($\beta_1=0.5, \beta_2=0.999$) \cite{kingma2014adam} with a decay of $1\cdot10^{-3}$. We set the learning rate at $1.2\cdot 10^{-4}$ for the feature generator $f_G(\cdot)$ and $5\cdot10^{-4}$ for our discriminators. We adopt the learning rate $4\cdot10^{-2}$ for the proxies as suggested in \cite{roth2022non}. For most of the experiments, we fixed the batch size to 90 as a default setting, which is consistent with \cite{kim2020proxy}. Empirically we apply \I{batch normalization} on the domain-level discriminator to reduce its correlation variance within the batch. For all experiments, the first layer of $f _C(\cdot)$ is set to 512. For the second layer, we assigned 128 dimensions to the CUB200 and CARS196 datasets, 8192 dimensions to the SOP datasets, and 4096 dimensions to the In-Shop datasets. We set $\{\eta = 0.005, \gamma= 0.0075\}$ for CUB200, and $\{\eta = 0.01, \gamma=0.0075\}$ for CARS196. We select $\{\eta = 0.01, \gamma= 0.005\}$ for both SOP and In-Shop datasets.

\subsection{Qualitative Results}
\B{Comparing with Proxy Baselines.} We compare the performance of our approach with the existing proxy-based metric learning methods and the recent \I{state-of-the-art} metric learning methods on the popular benchmarks introduced above (refer to Table \ref{table:sota_1}). We observe that our DADA frameworks can significantly improve the performance of the original proxy-based DML methods (marked with $\vartriangleright$) by a large margin. Specifically, comparing with the original PA method on ResNet50, our proposed PA+DADA outperforms $3.2pp$ ($4.6\%$) on the recall@1 of CUB200 and $4.4pp$ ($5.0\%$) on the recall@1 of CARS196. On the larger datasets (SOP and In-Shop), our method is also better than the original PA and PNCA++.

\B{Comparing with state-of-the-art.} We further compare the performance of our method with the state-of-the-art methods based on the CNN backbones as listed in Table \ref{table:sota_1} and \ref{table:sota_Inshop}. For the CARS196 dataset, our method reaches $92.1$ on Recall@1, which has a $0.9pp$ improvement over the previous state-of-the-art AVSL \cite{zhang2022attributable} on the ResNet50 backbone. For CUB200, our method outperforms the previous state-of-the-art AVSL $0.6pp$ on Recall@1, $0.2pp$ on Recall@2, and $0.1pp$ on Recall@4. We observe that our performance on SOP and In-Shop is limited but very close to the previous state-of-the-art IBC \cite{ibc2021}, CRT \cite{kan2022coded}, and HIST \cite{lim2022hypergraph} on a few metrics. The lesser improvement in the high-value recall of these two datasets is mainly due to the large number of classes (11318 and 3997) and the limited number of samples in each class (less than 10).  This causes some difficulty for our category-level discriminator to learn the discriminative information. Nevertheless, our method still achieves good performance comparable to those of the state-of-the-art methods in all metrics and outperforms other proxy-related methods on these two datasets. We will investigate techniques to overcome this limitation in our future works.


\subsection{Ablation Study}

\B{Contributions of the Objective Components.} We analyze the ablation study to evaluate the contribution of each objective component of our proposed framework based on both ProxyAnchor \cite{kim2020proxy} and ProxyNCA++ \cite{teh2020proxynca++} on the CUB200 as listed in Table \ref{table:ablation}. We first notice that the data augmentation strategy (Aug) does not improve our baseline significantly in the absence of $\mathcal{L}_{adv}$ and $\mathcal{L}_{cls}$. This is because, without those regularization losses, Aug simply boosts some redundant positive samples and the mixed features do not take part in training. We conclude that the domain-level discriminator with $\mathcal{L}_{adv}$ has higher efficiency when the category-level discriminator with $\mathcal{L}_{cls}$ helps regularize the space and avoid the inter-class ambiguity. It increases the improvement to $+2.3pp$ on R@1 and $+1.7pp$ on MAP@R from $+1.1pp$ on R@1 and $+0.8pp$ on MAP@R in comparison with the single $\mathcal{L}_{adv}$ setting. We also demonstrate that the efficiency of the category-level classifier ($+\mathcal{L}_{cls}$) can be further improved by comparing the discrepancy of class prediction ($+\mathcal{L}_d$) between the source data and target proxies in adversarial learning. Comparing the general discrepancy L1 distance ($+\mathcal{L}_1$), the proposed NWD also shows increasing performances on both R@1 and MAP@R. A similar conclusion can also be driven by the results based on ProxyNCA. Therefore, we conclude that the combination of the domain and the category-level discriminator is more suitable for proxy-based DML than the settings with any single discriminator. We also study the impact of our hyperparameters and the combination of data groups that apply domain adaptation in the Appendix.


\section{Conclusion}
\label{sec:conclusion}
In this paper, we present an adversarial domain adaptation method with data augmentation to optimize the hidden space of the data and the proxies. We overcome the initial distribution gap between them to boost the learning efficiency of deep metric learning. We propose to align the domains of the data and the initial proxies by optimizing two classifiers at different levels, and training the embedding function and the proxies against them. To enhance the density of the manifold, we propose a strategy to conduct a mixture space by mixing the features from both domains. Our experimental results based on four popular deep metric learning benchmarks demonstrate that our learning method and mixed space efficiently boost the learning efficiency of existing proxy-based methods. While our framework focuses on solving the challenge of proxy-based DML methods, we believe it can be easily extended to other related metric learning methods, and it can also benefit zero-shot and self-supervised learning works. These are interesting and challenging works for future study.

\newpage
{
\bibliography{aaai24}

\begin{thebibliography}{82}
\providecommand{\natexlab}[1]{#1}

\bibitem[{Ben-David et~al.(2010)Ben-David, Blitzer, Crammer, Kulesza, Pereira,
  and Vaughan}]{ben2010theory}
Ben-David, S.; Blitzer, J.; Crammer, K.; Kulesza, A.; Pereira, F.; and Vaughan,
  J.~W. 2010.
\newblock A theory of learning from different domains.
\newblock \emph{Machine learning}, 79: 151--175.

\bibitem[{Boudiaf et~al.(2020)Boudiaf, Rony, Ziko, Granger, Pedersoli,
  Piantanida, and Ayed}]{boudiaf2020unifying}
Boudiaf, M.; Rony, J.; Ziko, I.~M.; Granger, E.; Pedersoli, M.; Piantanida, P.;
  and Ayed, I.~B. 2020.
\newblock A unifying mutual information view of metric learning: cross-entropy
  vs. pairwise losses.
\newblock In \emph{ECCV}, 548--564. Springer.

\bibitem[{Che et~al.(2017)Che, Li, Jacob, Bengio, and Li}]{che2016mode}
Che, T.; Li, Y.; Jacob, A.~P.; Bengio, Y.; and Li, W. 2017.
\newblock Mode regularized generative adversarial networks.
\newblock \emph{ICLR}.

\bibitem[{Chen et~al.(2022)Chen, Chen, Wei, Jin, Tan, Jin, and
  Chen}]{chen2022reusing}
Chen, L.; Chen, H.; Wei, Z.; Jin, X.; Tan, X.; Jin, Y.; and Chen, E. 2022.
\newblock Reusing the Task-specific Classifier as a Discriminator:
  Discriminator-free Adversarial Domain Adaptation.
\newblock In \emph{CVPR}, 7181--7190.

\bibitem[{Cheng et~al.(2016)Cheng, Gong, Zhou, Wang, and
  Zheng}]{cheng2016person}
Cheng, D.; Gong, Y.; Zhou, S.; Wang, J.; and Zheng, N. 2016.
\newblock Person re-identification by multi-channel parts-based cnn with
  improved triplet loss function.
\newblock In \emph{CVPR}, 1335--1344.

\bibitem[{Chopra, Hadsell, and LeCun(2005)}]{chopra2005learning}
Chopra, S.; Hadsell, R.; and LeCun, Y. 2005.
\newblock Learning a similarity metric discriminatively, with application to
  face verification.
\newblock In \emph{CVPR}, volume~1, 539--546. IEEE.

\bibitem[{Cui et~al.(2020)Cui, Wang, Zhuo, Li, Huang, and
  Tian}]{cui2020towards}
Cui, S.; Wang, S.; Zhuo, J.; Li, L.; Huang, Q.; and Tian, Q. 2020.
\newblock Towards discriminability and diversity: Batch nuclear-norm
  maximization under label insufficient situations.
\newblock In \emph{CVPR}, 3941--3950.

\bibitem[{Dai et~al.(2019)Dai, Chen, Gu, Zhu, and Tan}]{dai2019batch}
Dai, Z.; Chen, M.; Gu, X.; Zhu, S.; and Tan, P. 2019.
\newblock Batch dropblock network for person re-identification and beyond.
\newblock In \emph{ICCV}, 3691--3701.

\bibitem[{Deng et~al.(2009)Deng, Dong, Socher, Li, Li, and
  Fei-Fei}]{deng2009imagenet}
Deng, J.; Dong, W.; Socher, R.; Li, L.-J.; Li, K.; and Fei-Fei, L. 2009.
\newblock Imagenet: A large-scale hierarchical image database.
\newblock In \emph{CVPR}, 248--255. Ieee.

\bibitem[{El-Nouby et~al.(2021)El-Nouby, Neverova, Laptev, and
  J{\'e}gou}]{el2021vit}
El-Nouby, A.; Neverova, N.; Laptev, I.; and J{\'e}gou, H. 2021.
\newblock Training vision transformers for image retrieval.
\newblock \emph{arXiv preprint arXiv:2102.05644}.

\bibitem[{Ermolov et~al.(2022)Ermolov, Mirvakhabova, Khrulkov, Sebe, and
  Oseledets}]{ermolov2022bolic}
Ermolov, A.; Mirvakhabova, L.; Khrulkov, V.; Sebe, N.; and Oseledets, I. 2022.
\newblock Hyperbolic Vision Transformers: Combining Improvements in Metric
  Learning.
\newblock In \emph{CVPR}, 7409--7419.

\bibitem[{Ganin and Lempitsky(2015)}]{ganin2015unsupervised}
Ganin, Y.; and Lempitsky, V. 2015.
\newblock Unsupervised domain adaptation by backpropagation.
\newblock In \emph{ICML}, 1180--1189. PMLR.

\bibitem[{Ganin et~al.(2016)Ganin, Ustinova, Ajakan, Germain, Larochelle,
  Laviolette, Marchand, and Lempitsky}]{ganin2016domain}
Ganin, Y.; Ustinova, E.; Ajakan, H.; Germain, P.; Larochelle, H.; Laviolette,
  F.; Marchand, M.; and Lempitsky, V. 2016.
\newblock Domain-adversarial training of neural networks.
\newblock \emph{The journal of machine learning research}, 17(1): 2096--2030.

\bibitem[{Ge et~al.(2020)Ge, Wang, Zhu, Zhao, and Li}]{ge2020self}
Ge, Y.; Wang, H.; Zhu, F.; Zhao, R.; and Li, H. 2020.
\newblock Self-supervising fine-grained region similarities for large-scale
  image localization.
\newblock In \emph{ECCV}, 369--386. Springer.

\bibitem[{Goldberger et~al.(2004)Goldberger, Hinton, Roweis, and
  Salakhutdinov}]{goldberger2004neighbourhood}
Goldberger, J.; Hinton, G.~E.; Roweis, S.; and Salakhutdinov, R.~R. 2004.
\newblock Neighbourhood components analysis.
\newblock \emph{NIPS}, 17.

\bibitem[{Goodfellow et~al.(2020)Goodfellow, Pouget-Abadie, Mirza, Xu,
  Warde-Farley, Ozair, Courville, and Bengio}]{goodfellow2020generative}
Goodfellow, I.; Pouget-Abadie, J.; Mirza, M.; Xu, B.; Warde-Farley, D.; Ozair,
  S.; Courville, A.; and Bengio, Y. 2020.
\newblock Generative adversarial networks.
\newblock \emph{Communications of the ACM}, 63(11): 139--144.

\bibitem[{Hadsell, Chopra, and LeCun(2006)}]{hadsell2006dimensionality}
Hadsell, R.; Chopra, S.; and LeCun, Y. 2006.
\newblock Dimensionality reduction by learning an invariant mapping.
\newblock In \emph{CVPR}, volume~2, 1735--1742. IEEE.

\bibitem[{HassanPour~Zonoozi and Seydi(2022)}]{hassanpour2022survey}
HassanPour~Zonoozi, M.; and Seydi, V. 2022.
\newblock A Survey on Adversarial Domain Adaptation.
\newblock \emph{Neural Processing Letters}, 1--41.

\bibitem[{He et~al.(2016)He, Zhang, Ren, and Sun}]{he2016deep}
He, K.; Zhang, X.; Ren, S.; and Sun, J. 2016.
\newblock Deep residual learning for image recognition.
\newblock In \emph{CVPR}, 770--778.

\bibitem[{Hermans, Beyer, and Leibe(2017)}]{hermans2017defense}
Hermans, A.; Beyer, L.; and Leibe, B. 2017.
\newblock In defense of the triplet loss for person re-identification.
\newblock \emph{arXiv preprint arXiv:1703.07737}.

\bibitem[{Ioffe and Szegedy(2015)}]{ioffe2015batch}
Ioffe, S.; and Szegedy, C. 2015.
\newblock Batch normalization: Accelerating deep network training by reducing
  internal covariate shift.
\newblock In \emph{ICML}, 448--456. PMLR.

\bibitem[{Isola et~al.(2017)Isola, Zhu, Zhou, and Efros}]{isola2017image}
Isola, P.; Zhu, J.-Y.; Zhou, T.; and Efros, A.~A. 2017.
\newblock Image-to-image translation with conditional adversarial networks.
\newblock In \emph{CVPR}, 1125--1134.

\bibitem[{Kan et~al.(2022)Kan, Liang, Li, Cen, Wang, and He}]{kan2022coded}
Kan, S.; Liang, Y.; Li, M.; Cen, Y.; Wang, J.; and He, Z. 2022.
\newblock Coded Residual Transform for Generalizable Deep Metric Learning.
\newblock \emph{NeurIPS}.

\bibitem[{Katharopoulos and Fleuret(2018)}]{katharopoulos2018not}
Katharopoulos, A.; and Fleuret, F. 2018.
\newblock Not all samples are created equal: Deep learning with importance
  sampling.
\newblock In \emph{ICML}, 2525--2534. PMLR.

\bibitem[{Kim et~al.(2020)Kim, Kim, Cho, and Kwak}]{kim2020proxy}
Kim, S.; Kim, D.; Cho, M.; and Kwak, S. 2020.
\newblock Proxy anchor loss for deep metric learning.
\newblock In \emph{CVPR}, 3238--3247.

\bibitem[{Kingma and Ba(2015)}]{kingma2014adam}
Kingma, D.~P.; and Ba, J. 2015.
\newblock Adam: A method for stochastic optimization.
\newblock \emph{ICLR}.

\bibitem[{Kingma and Welling(2013)}]{kingma2013auto}
Kingma, D.~P.; and Welling, M. 2013.
\newblock Auto-encoding variational bayes.
\newblock \emph{arXiv preprint arXiv:1312.6114}.

\bibitem[{Ko, Gu, and Kim(2021)}]{ko2021learning}
Ko, B.; Gu, G.; and Kim, H.-G. 2021.
\newblock Learning with memory-based virtual classes for deep metric learning.
\newblock In \emph{CVPR}, 11792--11801.

\bibitem[{Krause et~al.(2013)Krause, Stark, Deng, and Fei-Fei}]{krause20133d}
Krause, J.; Stark, M.; Deng, J.; and Fei-Fei, L. 2013.
\newblock 3d object representations for fine-grained categorization.
\newblock In \emph{ICCV workshop}, 554--561.

\bibitem[{Laradji and Babanezhad(2020)}]{laradji2020m}
Laradji, I.~H.; and Babanezhad, R. 2020.
\newblock M-adda: Unsupervised domain adaptation with deep metric learning.
\newblock \emph{Domain adaptation for visual understanding}, 17--31.

\bibitem[{Lee et~al.(2019)Lee, Batra, Baig, and Ulbricht}]{lee2019sliced}
Lee, C.-Y.; Batra, T.; Baig, M.~H.; and Ulbricht, D. 2019.
\newblock Sliced wasserstein discrepancy for unsupervised domain adaptation.
\newblock In \emph{CVPR}, 10285--10295.

\bibitem[{Lee, Jin, and Jain(2008)}]{lee2008rank}
Lee, J.-E.; Jin, R.; and Jain, A.~K. 2008.
\newblock Rank-based distance metric learning: An application to image
  retrieval.
\newblock In \emph{CVPR}, 1--8. IEEE.

\bibitem[{Lim et~al.(2022)Lim, Yun, Park, and Choi}]{lim2022hypergraph}
Lim, J.; Yun, S.; Park, S.; and Choi, J.~Y. 2022.
\newblock Hypergraph-Induced Semantic Tuplet Loss for Deep Metric Learning.
\newblock In \emph{CVPR}, 212--222.

\bibitem[{Liu et~al.(2022)Liu, Huang, Zhuang, Yang, Tan, and
  Wang}]{liu2022densely}
Liu, L.; Huang, S.; Zhuang, Z.; Yang, R.; Tan, M.; and Wang, Y. 2022.
\newblock DAS: Densely-Anchored Sampling for Deep Metric Learning.
\newblock \emph{ECCV}.

\bibitem[{Liu et~al.(2016)Liu, Luo, Qiu, Wang, and Tang}]{liu2016deepfashion}
Liu, Z.; Luo, P.; Qiu, S.; Wang, X.; and Tang, X. 2016.
\newblock Deepfashion: Powering robust clothes recognition and retrieval with
  rich annotations.
\newblock In \emph{CVPR}, 1096--1104.

\bibitem[{Lu et~al.(2015)Lu, Yan, Ren, Song, Sebe, and
  Kambhamettu}]{lu2015localize}
Lu, G.; Yan, Y.; Ren, L.; Song, J.; Sebe, N.; and Kambhamettu, C. 2015.
\newblock Localize me anywhere, anytime: a multi-task point-retrieval approach.
\newblock In \emph{ICCV}, 2434--2442.

\bibitem[{Milbich et~al.(2020)Milbich, Roth, Bharadhwaj, Sinha, Bengio, Ommer,
  and Cohen}]{milbich2020diva}
Milbich, T.; Roth, K.; Bharadhwaj, H.; Sinha, S.; Bengio, Y.; Ommer, B.; and
  Cohen, J.~P. 2020.
\newblock Diva: Diverse visual feature aggregation for deep metric learning.
\newblock In \emph{ECCV}, 590--607. Springer.

\bibitem[{Movshovitz-Attias et~al.(2017)Movshovitz-Attias, Toshev, Leung,
  Ioffe, and Singh}]{movshovitz2017no}
Movshovitz-Attias, Y.; Toshev, A.; Leung, T.~K.; Ioffe, S.; and Singh, S. 2017.
\newblock No fuss distance metric learning using proxies.
\newblock In \emph{CVPR}, 360--368.

\bibitem[{Musgrave, Belongie, and Lim(2020)}]{musgrave2020metric}
Musgrave, K.; Belongie, S.; and Lim, S.-N. 2020.
\newblock A metric learning reality check.
\newblock In \emph{ECCV}, 681--699. Springer.

\bibitem[{Oh~Song et~al.(2017)Oh~Song, Jegelka, Rathod, and
  Murphy}]{oh2017deep}
Oh~Song, H.; Jegelka, S.; Rathod, V.; and Murphy, K. 2017.
\newblock Deep metric learning via facility location.
\newblock In \emph{CVPR}, 5382--5390.

\bibitem[{Oh~Song et~al.(2016)Oh~Song, Xiang, Jegelka, and
  Savarese}]{oh2016deep}
Oh~Song, H.; Xiang, Y.; Jegelka, S.; and Savarese, S. 2016.
\newblock Deep metric learning via lifted structured feature embedding.
\newblock In \emph{CVPR}, 4004--4012.

\bibitem[{Patel, Tolias, and Matas(2022)}]{patel2022recall}
Patel, Y.; Tolias, G.; and Matas, J. 2022.
\newblock Recall@ k surrogate loss with large batches and similarity mixup.
\newblock In \emph{CVPR}, 7502--7511.

\bibitem[{Pei et~al.(2018)Pei, Cao, Long, and Wang}]{pei2018multi}
Pei, Z.; Cao, Z.; Long, M.; and Wang, J. 2018.
\newblock Multi-adversarial domain adaptation.
\newblock In \emph{AAAI}.

\bibitem[{Pinheiro(2018)}]{pinheiro2018unsupervised}
Pinheiro, P.~O. 2018.
\newblock Unsupervised domain adaptation with similarity learning.
\newblock In \emph{CVPR}, 8004--8013.

\bibitem[{Qian et~al.(2019)Qian, Shang, Sun, Hu, Li, and
  Jin}]{qian2019softtriple}
Qian, Q.; Shang, L.; Sun, B.; Hu, J.; Li, H.; and Jin, R. 2019.
\newblock Softtriple loss: Deep metric learning without triplet sampling.
\newblock In \emph{ICCV}, 6450--6458.

\bibitem[{Qui{\~n}onero-Candela et~al.(2008)Qui{\~n}onero-Candela, Sugiyama,
  Schwaighofer, and Lawrence}]{quinonero2008covariate}
Qui{\~n}onero-Candela, J.; Sugiyama, M.; Schwaighofer, A.; and Lawrence, N.
  2008.
\newblock Covariate shift and local learning by distribution matching.

\bibitem[{Radford, Metz, and Chintala(2015)}]{radford2015unsupervised}
Radford, A.; Metz, L.; and Chintala, S. 2015.
\newblock Unsupervised representation learning with deep convolutional
  generative adversarial networks.
\newblock \emph{arXiv preprint arXiv:1511.06434}.

\bibitem[{Ren and Hua(2018)}]{ren2018improved}
Ren, L.; and Hua, K. 2018.
\newblock Improved image description via embedded object structure graph and
  semantic feature matching.
\newblock In \emph{ISM}, 73--80. IEEE.

\bibitem[{Ren et~al.(2021)Ren, Li, Wang, and Hua}]{ren2021beyond}
Ren, L.; Li, K.; Wang, L.; and Hua, K. 2021.
\newblock Beyond the deep metric learning: enhance the cross-modal matching
  with adversarial discriminative domain regularization.
\newblock In \emph{ICPR}, 10165--10172. IEEE.

\bibitem[{Ren, Qi, and Hua(2019)}]{ren2019improving}
Ren, L.; Qi, G.-J.; and Hua, K. 2019.
\newblock Improving diversity of image captioning through variational
  autoencoders and adversarial learning.
\newblock In \emph{WACV}, 263--272. IEEE.

\bibitem[{Rippel et~al.(2016)Rippel, Paluri, Dollar, and
  Bourdev}]{rippel2015metric}
Rippel, O.; Paluri, M.; Dollar, P.; and Bourdev, L. 2016.
\newblock Metric learning with adaptive density discrimination.
\newblock \emph{ICLR}.

\bibitem[{Roth et~al.(2021)Roth, Milbich, Ommer, Cohen, and
  Ghassemi}]{roth2021simultaneous}
Roth, K.; Milbich, T.; Ommer, B.; Cohen, J.~P.; and Ghassemi, M. 2021.
\newblock Simultaneous similarity-based self-distillation for deep metric
  learning.
\newblock In \emph{ICML}, 9095--9106. PMLR.

\bibitem[{Roth et~al.(2020)Roth, Milbich, Sinha, Gupta, Ommer, and
  Cohen}]{roth2020revisiting}
Roth, K.; Milbich, T.; Sinha, S.; Gupta, P.; Ommer, B.; and Cohen, J.~P. 2020.
\newblock Revisiting training strategies and generalization performance in deep
  metric learning.
\newblock In \emph{ICML}, 8242--8252. PMLR.

\bibitem[{Roth, Vinyals, and Akata(2022)}]{roth2022non}
Roth, K.; Vinyals, O.; and Akata, Z. 2022.
\newblock Non-isotropy Regularization for Proxy-based Deep Metric Learning.
\newblock In \emph{CVPR}, 7420--7430.

\bibitem[{Roweis, Hinton, and Salakhutdinov(2004)}]{roweis2004neighbourhood}
Roweis, S.; Hinton, G.; and Salakhutdinov, R. 2004.
\newblock Neighbourhood component analysis.
\newblock \emph{NIPS}, 17(513-520): 4.

\bibitem[{Saito et~al.(2018)Saito, Watanabe, Ushiku, and
  Harada}]{saito2018maximum}
Saito, K.; Watanabe, K.; Ushiku, Y.; and Harada, T. 2018.
\newblock Maximum classifier discrepancy for unsupervised domain adaptation.
\newblock In \emph{CVPR}, 3723--3732.

\bibitem[{Schroff, Kalenichenko, and Philbin(2015)}]{schroff2015facenet}
Schroff, F.; Kalenichenko, D.; and Philbin, J. 2015.
\newblock Facenet: A unified embedding for face recognition and clustering.
\newblock In \emph{CVPR}, 815--823.

\bibitem[{Seidenschwarz, Elezi, and Leal-Taix{\'e}(2021)}]{ibc2021}
Seidenschwarz, J.~D.; Elezi, I.; and Leal-Taix{\'e}, L. 2021.
\newblock Learning intra-batch connections for deep metric learning.
\newblock In \emph{ICML}, 9410--9421. PMLR.

\bibitem[{Sohn(2016)}]{sohn2016improved}
Sohn, K. 2016.
\newblock Improved deep metric learning with multi-class n-pair loss objective.
\newblock \emph{NIPS}, 29.

\bibitem[{Suh et~al.(2019)Suh, Han, Kim, and Lee}]{suh2019stochastic}
Suh, Y.; Han, B.; Kim, W.; and Lee, K.~M. 2019.
\newblock Stochastic class-based hard example mining for deep metric learning.
\newblock In \emph{CVPR}, 7251--7259.

\bibitem[{Tan, Yuan, and Ordonez(2021)}]{tan2021instance}
Tan, F.; Yuan, J.; and Ordonez, V. 2021.
\newblock Instance-level image retrieval using reranking transformers.
\newblock In \emph{ICCV}, 12105--12115.

\bibitem[{Teh, DeVries, and Taylor(2020)}]{teh2020proxynca++}
Teh, E.~W.; DeVries, T.; and Taylor, G.~W. 2020.
\newblock Proxynca++: Revisiting and revitalizing proxy neighborhood component
  analysis.
\newblock In \emph{ECCV}, 448--464. Springer.

\bibitem[{Torralba and Efros(2011)}]{torralba2011unbiased}
Torralba, A.; and Efros, A.~A. 2011.
\newblock Unbiased look at dataset bias.
\newblock In \emph{CVPR}, 1521--1528. IEEE.

\bibitem[{Tzeng et~al.(2017)Tzeng, Hoffman, Saenko, and
  Darrell}]{tzeng2017adversarial}
Tzeng, E.; Hoffman, J.; Saenko, K.; and Darrell, T. 2017.
\newblock Adversarial discriminative domain adaptation.
\newblock In \emph{CVPR}, 7167--7176.

\bibitem[{Venkataramanan et~al.(2022)Venkataramanan, Psomas, Avrithis, Kijak,
  Amsaleg, and Karantzalos}]{mix2021}
Venkataramanan, S.; Psomas, B.; Avrithis, Y.; Kijak, E.; Amsaleg, L.; and
  Karantzalos, K. 2022.
\newblock It takes two to tango: Mixup for deep metric learning.
\newblock \emph{ICLR}.

\bibitem[{Wah et~al.(2011)Wah, Branson, Welinder, Perona, and
  Belongie}]{wah2011caltech}
Wah, C.; Branson, S.; Welinder, P.; Perona, P.; and Belongie, S. 2011.
\newblock The caltech-ucsd birds-200-2011 dataset.

\bibitem[{Wang et~al.(2017)Wang, Yang, Xu, Hanjalic, and
  Shen}]{wang2017adversarial}
Wang, B.; Yang, Y.; Xu, X.; Hanjalic, A.; and Shen, H.~T. 2017.
\newblock Adversarial cross-modal retrieval.
\newblock In \emph{Multimedia}, 154--162.

\bibitem[{Wang et~al.(2014)Wang, Song, Leung, Rosenberg, Wang, Philbin, Chen,
  and Wu}]{wang2014learning}
Wang, J.; Song, Y.; Leung, T.; Rosenberg, C.; Wang, J.; Philbin, J.; Chen, B.;
  and Wu, Y. 2014.
\newblock Learning fine-grained image similarity with deep ranking.
\newblock In \emph{CVPR}, 1386--1393.

\bibitem[{Wang and Deng(2018)}]{wang2018deep}
Wang, M.; and Deng, W. 2018.
\newblock Deep visual domain adaptation: A survey.
\newblock \emph{Neurocomputing}, 312: 135--153.

\bibitem[{Wang et~al.(2019)Wang, Han, Huang, Dong, and Scott}]{wang2019multi}
Wang, X.; Han, X.; Huang, W.; Dong, D.; and Scott, M.~R. 2019.
\newblock Multi-similarity loss with general pair weighting for deep metric
  learning.
\newblock In \emph{CVPR}, 5022--5030.

\bibitem[{Wang et~al.(2020)Wang, Zhang, Huang, and Scott}]{wang2020cross}
Wang, X.; Zhang, H.; Huang, W.; and Scott, M.~R. 2020.
\newblock Cross-batch memory for embedding learning.
\newblock In \emph{CVPR}, 6388--6397.

\bibitem[{Weinberger and Saul(2009)}]{weinberger2009distance}
Weinberger, K.~Q.; and Saul, L.~K. 2009.
\newblock Distance metric learning for large margin nearest neighbor
  classification.
\newblock \emph{Journal of machine learning research}, 10(2).

\bibitem[{Wojke and Bewley(2018)}]{wojke2018deep}
Wojke, N.; and Bewley, A. 2018.
\newblock Deep cosine metric learning for person re-identification.
\newblock In \emph{WACV}, 748--756. IEEE.

\bibitem[{Wu et~al.(2017)Wu, Manmatha, Smola, and Krahenbuhl}]{wu2017sampling}
Wu, C.-Y.; Manmatha, R.; Smola, A.~J.; and Krahenbuhl, P. 2017.
\newblock Sampling matters in deep embedding learning.
\newblock In \emph{CVPR}, 2840--2848.

\bibitem[{Xuan et~al.(2020)Xuan, Stylianou, Liu, and Pless}]{xuan2020hard}
Xuan, H.; Stylianou, A.; Liu, X.; and Pless, R. 2020.
\newblock Hard negative examples are hard, but useful.
\newblock In \emph{ECCV}, 126--142. Springer.

\bibitem[{Yang et~al.(2018)Yang, She, Lai, and Yang}]{yang2018retrieving}
Yang, J.; She, D.; Lai, Y.-K.; and Yang, M.-H. 2018.
\newblock Retrieving and classifying affective images via deep metric learning.
\newblock In \emph{AAAI}, volume~32.

\bibitem[{Yang et~al.(2022)Yang, Bastan, Zhu, Gray, and
  Samaras}]{yang2022hierarchical}
Yang, Z.; Bastan, M.; Zhu, X.; Gray, D.; and Samaras, D. 2022.
\newblock Hierarchical proxy-based loss for deep metric learning.
\newblock In \emph{WACV}, 1859--1868.

\bibitem[{Yi et~al.(2014)Yi, Lei, Liao, and Li}]{yi2014deep}
Yi, D.; Lei, Z.; Liao, S.; and Li, S.~Z. 2014.
\newblock Deep metric learning for person re-identification.
\newblock In \emph{ICPR}, 34--39. IEEE.

\bibitem[{Zhang et~al.(2022)Zhang, Zheng, Zhou, and Lu}]{zhang2022attributable}
Zhang, B.; Zheng, W.; Zhou, J.; and Lu, J. 2022.
\newblock Attributable Visual Similarity Learning.
\newblock In \emph{CVPR}, 7532--7541.

\bibitem[{Zheng et~al.(2021{\natexlab{a}})Zheng, Wang, Lu, and
  Zhou}]{zheng2021dcml}
Zheng, W.; Wang, C.; Lu, J.; and Zhou, J. 2021{\natexlab{a}}.
\newblock Deep compositional metric learning.
\newblock In \emph{CVPR}, 9320--9329.

\bibitem[{Zheng et~al.(2021{\natexlab{b}})Zheng, Zhang, Lu, and
  Zhou}]{zheng2021drml}
Zheng, W.; Zhang, B.; Lu, J.; and Zhou, J. 2021{\natexlab{b}}.
\newblock Deep relational metric learning.
\newblock In \emph{ICCV}, 12065--12074.

\bibitem[{Zhu et~al.(2020)Zhu, Yang, Deng, and Liu}]{zhu2020fewer}
Zhu, Y.; Yang, M.; Deng, C.; and Liu, W. 2020.
\newblock Fewer is more: A deep graph metric learning perspective using fewer
  proxies.
\newblock \emph{NeurIPS}, 33: 17792--17803.

\end{thebibliography}
}


\newpage
\clearpage

\section{Appendix}

\appendix
\setcounter{secnumdepth}{1}

In the appendix, we first discuss the generalized bound of our domain adaptation framework in Section \ref{appendix:theory}. Then, we compare our method with other similar methods in Section \ref{appendix:related}. We discuss the possible extension of our method to other model backbones in Section \ref{appendix:extension}. We also discuss some related studies including the data domain and the discriminators in Section \ref{appendix:additional}. After that, we list detailed experimental settings in Section \ref{appendix:setting} and describe some existing limitations in Section \ref{appendix:limits}. We illustrate the overall architecture of our method in Figure \ref{fig:framework} and show some examples of image retrieval results.

\section{Theory Insight}
\label{appendix:theory}

	\subsection{Notation and Definitions} 
	
	We adopt the definition \I{domain} as the distribution $\mathcal{D}$ on the space $\mathcal{X}$ with the ground truth labeling function $y$. In our scenario, we have source domain $\{\mathcal{D}_S, y_S\}$, the intermediate domain $\{\mathcal{D}_M, y_M\}$ and the target domain $\{\mathcal{D}_T, y_T\}$. The \I{hypothesis} $h$ is a function that we search to demonstrate the labeling $y$. For our data sample $x \sim \mathcal{D}_S$ and its corresponding proxy $p \sim \mathcal{D}_T$, the \I{risk} of the hypothesis $h$ is defined as the error risk with the labeling function $y$:
	\begin{equation}		
		\epsilon_S (h, y) := \mathbb{E}_{x \thicksim \mathcal{D}_S} [|h(x) - y_x|]
	\end{equation}
	\begin{equation}		
		\epsilon_T (h, y') := \mathbb{E}_{p \thicksim \mathcal{D}_T} [|h(p) - y'_p|],
	\end{equation}
	where $y'$ is the hidden labeling function of the target domain. Thus, our goal is to find the bound of $\epsilon_T (h, y')$ in the target domain where the \I{domain adaptation} is summarized to minimize the target risk $\epsilon_T(h, y')$ in term of the source risk $\epsilon_S (h, y)$ and other terms that affect its upper bound.
	
	Given the set of hypothesis class $h\in\mathcal{H}$, Ben-David et al. \cite{ben2010theory} define the symmetric difference hypothesis space $\mathcal{H} \Delta \mathcal{H}$, and the domain divergence $d_{\mathcal{H}\Delta\mathcal{H}}(\mathcal{D}_S, \mathcal{D}_T)$ where $\mathcal{H} \Delta \mathcal{H} := \{h(x) \neq h'(x)| h, h' \in \mathcal{H} \}$. Thus $d_{\mathcal{H}\Delta\mathcal{H}}(\mathcal{D}_S, \mathcal{D}_T)$ can be defined as,
	\begin{align*}
		d_{\mathcal{H}\Delta\mathcal{H}}(\mathcal{D}_S, \mathcal{D}_T) &= 2\sup_{h, h' \in \mathcal{H}\Delta\mathcal{H}} 
		|P_{\mathcal{D}_S}(h \neq h') - P_{\mathcal{D}_T}(h \neq h') | \\
		& \geq 2|\epsilon_S(h, h') - \epsilon_T(h, h')|
	\end{align*}
	
	To estimate domain divergence within the space of finite samples $(\mathcal{U}_S, \mathcal{U}_T)$, the $d_{\mathcal{H}\Delta\mathcal{H}}(\mathcal{D}_S, \mathcal{D}_T)$ is further relaxed to
	\begin{equation}
		d_{\mathcal{H}\Delta\mathcal{H}}(\mathcal{D}_S, \mathcal{D}_T) \leq d_{\mathcal{H}\Delta\mathcal{H}}(\mathcal{U}_S, \mathcal{U}_T) 
		+ O(\sqrt{1/m})
	\end{equation}
	where $O(\sqrt{1/m})$ is a empirical function introduced in \cite{ben2010theory} with converge rate $\sqrt{1/m}$ and $m$ is the size of data samples $\mathcal{U}_S$ and $\mathcal{U}_T$.

	\subsection{The Generalization Bound} 
	
	With the definition of $d_{\mathcal{H}\Delta\mathcal{H}}$ introduced above, for every $h \in \mathcal{H}$ the bound for data space $\mathcal{U}_S$ and $\mathcal{U}_T$ can be described as,
	\begin{equation}
		\epsilon_T(h) \leq \epsilon_S(h) + \frac{1}{2}d_{\mathcal{H}\Delta\mathcal{H}}(\mathcal{U}_S, \mathcal{U}_T) + \lambda + O(\sqrt{1/m}),
		\label{bound}
	\end{equation}
	where $\lambda$ denotes the combined risk of optimized hypothesis $h^*$: $\lambda := \epsilon_S(h^*) + \epsilon_T(h^*)$. 
	
	Now we assume our source, target, and intermediate domain service three possible source-target adaptation pairs: $\{\mathcal{U}_S, \mathcal{U}_T\}, \{\mathcal{U}_S, \mathcal{U}_M\}$, and $\{\mathcal{U}_M, \mathcal{U}_T\}$. Note that $\mathcal{U}_T$ does not serve as the source domain since it initially does not contain any semantic information, and its labeling space is still hidden. We combine the single bound of each source-target pair, as a convex function, to a combined boundary. In other words, by combining the source domain $D_S$ and intermediate domain $D_M$ to a single domain, our generalization bound of risks on three pairs can be inferred as the following:
	\begin{theorem}
		Let $\mathcal{H}$ be a hypothesis space of VC-dimension $d$, and $\mathcal{U}_S, \mathcal{U}_M$ and  $\mathcal{U}_T$ are samples of size $m$ drawn from $\mathcal{D}_S, \mathcal{D}_M, \mathcal{D}_T$. Then $\forall h \in \mathcal{H}$,
		\begin{equation}
			\begin{aligned}
				\epsilon_T(h) &\leq \frac{1}{2}\epsilon_S(h) + \frac{1}{4}(d_{\mathcal{H}\Delta\mathcal{H}}(\mathcal{U}_S, \mathcal{U}_M, \mathcal{U}_T)) \\ 
				& + \tilde{\lambda} + \tilde{O}(\sqrt{1/m})
			\end{aligned}
		\end{equation}
		where $\tilde{\lambda}$ denotes the combined risk of optimal hypothesis $h^*$ that $\tilde{\lambda} := \epsilon_S(h^*) + \epsilon_M(h^*) + \epsilon_T(h^*)$, and $d_{\mathcal{H}\Delta\mathcal{H}}(\mathcal{U}_S, \mathcal{U}_M, \mathcal{U}_T):=d_{\mathcal{H}\Delta\mathcal{H}}(\mathcal{U}_S, \mathcal{U}_T) + d_{\mathcal{H}\Delta\mathcal{H}}(\mathcal{U}_M, \mathcal{U}_T) + d_{\mathcal{H}\Delta\mathcal{H}}(\mathcal{U}_S, \mathcal{U}_M)$
		\label{theorem}
	\end{theorem}
	
	\begin{proof}
		We have the bound proposed in Eq. \ref{bound} for each source-target domain pair. Assume all domains have a finite sample of size $m$; the bounds for all pairs can be listed as follows:
		\begin{align}
			\epsilon_T(h) \leq \epsilon_S(h) + \frac{1}{2}d_{\mathcal{H}\Delta\mathcal{H}}(\mathcal{U}_S, \mathcal{U}_T) + \lambda_1 + O(\sqrt{1/m}) \label{f1} \\
			\epsilon_T(h) \leq \epsilon_M(h) + \frac{1}{2}d_{\mathcal{H}\Delta\mathcal{H}}(\mathcal{U}_M, \mathcal{U}_T) + \lambda_2 + O(\sqrt{1/m}) \label{f2} \\
			\epsilon_M(h) \leq \epsilon_S(h) + \frac{1}{2}d_{\mathcal{H}\Delta\mathcal{H}}(\mathcal{U}_S, \mathcal{U}_M) + \lambda_3 + O(\sqrt{1/m}) \label{f3}
		\end{align}
		Then we combine the convex upper bounds \ref{f1}, \ref{f2} and \ref{f3} with a interpolate parameter $\alpha=1/3$ as follows,
		\begin{equation}
			\begin{split}
				2\alpha \epsilon_T(h) + \alpha \epsilon_M(h) &\leq 2\alpha \epsilon_S(h) + \alpha \epsilon_M(h) \\
				& + \frac{\alpha}{2}(d_{\mathcal{H}\Delta\mathcal{H}}(\mathcal{U}_S, \mathcal{U}_M, \mathcal{U}_T)) + \\
				& + \alpha (\lambda_1 + \lambda_2 + \lambda_3) + \tilde{O}(\sqrt{1/m}),
			\end{split}		
		\end{equation} 
		where $\tilde{O}(\sqrt{1/m})$ is a interpolated empirical function. We simply replace $\lambda$ with $\tilde{\lambda}:= \alpha (\lambda_1 + \lambda_2 + \lambda_3) = \epsilon_S(h^*) + \epsilon_M(h^*) + \epsilon_T(h^*)$ to get our combined upper bound.
	\end{proof}

	\subsection{Connection to Our Objective Function} 
	
	For each $d_{\mathcal{H}\Delta\mathcal{H}}(\mathcal{U}, \mathcal{U}')$, we know that
	\begin{equation}
		\begin{split}
			d_{\mathcal{H}\Delta\mathcal{H}}(\mathcal{U}, \mathcal{U}') &= 2\sup_{h, h' \in \mathcal{H}\Delta\mathcal{H}} |P_{\mathcal{U}}(h \neq h') - P_{\mathcal{U'}}(h \neq h') | \\
			& \leq 2\sup_{h \in \mathcal{H}} |P_{\mathcal{U}}(h = 0) - P_{\mathcal{U'}}(h = 1)| \\
			& = 2\sup_{h \in \mathcal{H}} (P_{\mathcal{U}}(h = 0) + P_{\mathcal{U'}}(h = 1) - 1)
		\end{split}
	\end{equation}
	
	
\begin{figure*}[pt]
	\vspace{-5pt}
	\centering
	\includegraphics[width=0.82\textwidth,height=0.25\textheight]{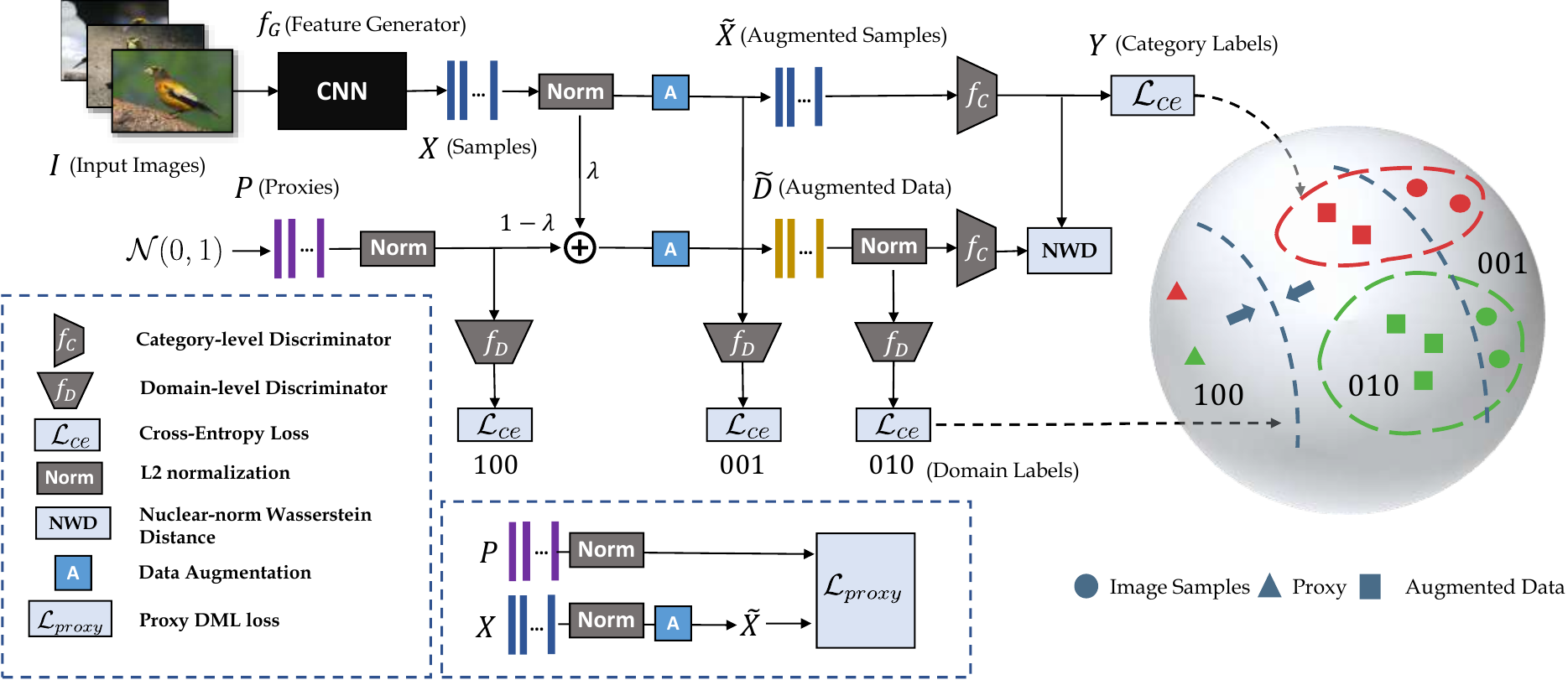}
	\caption{The overview of our framework. The input images are embedded with a CNN encoder. The proxies are randomly sampled and mixed with the embedding of image samples. Then a domain-level classifier $f_D(\cdot)$ and a category-level classifier $f_C(\cdot)$ are trained to predict the domain and class property of each sample and proxy. With the adversarial training paradigm, the features and proxies are moved to fool the $f_C(\cdot)$ and minimize the discrepancy of prediction of $f_D(\cdot)$. The dashed lines on the surface represent the surface boundary of our discriminators. The gradient from the adversarial learning pushes the data samples in the opposite direction from the separation of the boundary, which aligns the data and proxies.}
	\label{fig:framework}
	\vspace{-10pt}
\end{figure*}	

	
	Similarly, we can also relax our $h, h'$ from $\mathcal{H}\Delta\mathcal{H}$ space by specific class labels where $h=0, h=1, h=2$. Thus, we get the bound of $d_{\mathcal{H}\Delta\mathcal{H}}(\mathcal{U}_S, \mathcal{U}_M, \mathcal{U}_T)$ as follows,
	\begin{equation}
		\begin{aligned}
			&d_{\mathcal{H}\Delta\mathcal{H}}(\mathcal{U}_S, \mathcal{U}_M, \mathcal{U}_T) \leq \\ & 2\sup_{h \in \mathcal{H}} (P_{\mathcal{U_S}}(h = 0) + P_{\mathcal{U}_M}(h = 1) + P_{\mathcal{U_T}}(h = 2)),
		\end{aligned}
	\end{equation}
	When the combined risk $\tilde{\lambda}$ in \ref{theorem} is achieved by optimized hypothesis as a small constant, the bound mainly depends on the term $d_{\mathcal{H}\Delta\mathcal{H}}(\mathcal{U}_S, \mathcal{U}_M, \mathcal{U}_T)$ and $\epsilon_S(h)$ when $m$ is sufficiently large.
	
	The term $d_{\mathcal{H}\Delta\mathcal{H}}(\mathcal{U}_S, \mathcal{U}_M, \mathcal{U}_T)$ is the min-max optimize goal of our domain-level discriminator. And the category-level discriminator optimizes the risk $\epsilon_S(h)$. This shows that optimizing the risk in the target domain needs a combination of domain and category-level discriminators. Besides, the theory in \cite{chen2022reusing} also conduct that $|\epsilon_S(h, h') - \epsilon_T(h, h')| \leq d_{\mathcal{H}\Delta\mathcal{H}}(\mathcal{D}_S, \mathcal{D}_T)$ is also bounded by NWD with K-Lipschits constraint where $|\epsilon_S(h, h') - \epsilon_T(h, h')| \leq 2K \mathcal{L}_d(\nu_s, \nu_t)$. Here we explained the intuition behind our proposed objective.
	
	\subsection{Connection between Proxy Loss and $\epsilon_T(h)$} 
	
	Here, we try to roughly explain why making domain adaptation would help the proxy-based DML. We take Proxy-NCA loss as an example. Recall that Proxy-NCA has the following form:
	\begin{equation}
		\label{loss_panc}
		\mathcal{L}_{proxy}(X, P)=\sum_{x \in X}-\log \frac{e^{d\left(x, p^{+}\right)}}{\sum_{p^{-} \in P^{-}} e^{d\left(x, p^{-}\right)}},
	\end{equation}
	The NCA loss initially follows the design of the ``leave-one-out'' classification paradigm \cite{goldberger2004neighbourhood} where we try to label the sample $x$ and anchor $p$ into the same class (or to select $p$ as a neighbor of x with the same label). We know that optimizing the NCA loss maximizes the probability of labeling the sample $x$ and $p$ to the same class while pushing them from others. On the other hand, we have,
	\begin{equation}
		\begin{split}	
			& \frac{1}{2}d_{\mathcal{H}\Delta\mathcal{H}}(D_S, D_T) \geq |\epsilon_S(h(x)) - \epsilon_T(h(p))| \\
			& = |\mathbb{E}_{x \sim D_S}|h(x) - y| - \mathbb{E}_{p \sim D_T}|h(p) - y'||
		\end{split}
	\end{equation}
	By aligning distributions $(D_S, D_T)$ and labeling space $(y, y')$, we know that $\frac{1}{2}d_{\mathcal{H}\Delta\mathcal{H}}(D_S, D_T) \geq \mathbb{E}_{x, p \sim D}|h(x) - h(p)|$ which has the same optimizing target to the Proxy-NCA: to minimize the risk labeling $x$ and $p$ in different classes. Proxy-Anchor or other Proxy-based losses also follow this paradigm with modified NCA losses so that they initially consist of the same target.

\section{Comparison with Related Methods}
\label{appendix:related}
In this section, we discuss two related works, XBM \cite{wang2020cross} and MemVir \cite{ko2021learning}, that are intuitively close to the idea of our and other proxy-based DML works. Our approach and the XBM have something in common since both of us attempt to compare data representations with some other representations (XBM compares the data representations while ours compares class representations) that are unrelated to the training batch. MemVir, on the other hand, proposes to extend the space of the class by gradually adding the representations and class weights that are out of the batch. They all attempt to reduce the shift between the samples in batch and out of batch, which is also called \I{semantic drift}. However, there is a fundamental difference between proxy-based and XBM-related approaches, where proxy-based DML chooses to directly update the class representations in the whole data space while XBM approaches select to update a subset of data outside the batch.

\section{Extension to Other Backbones}
\label{appendix:extension}

Note that this paper compares our method with existing state-of-the-art proxy-based methods and other popular DML approaches that embed the images with CNN backbones (ResNet50 and InceptionBN). Some recent works that apply \I{Vision Transformers} \cite{el2021vit, tan2021instance}, which is pre-trained on larger datasets (ImageNet21K) or with extremely large batch sizes \cite{ermolov2022bolic, patel2022recall} are beyond the scope of this paper due to our limited computing resources. We anticipate that our method can be further extended with Transformer encoders (ViT) to boost its performance with larger GPU memory or parallel training with multi-GPUs in future works. We leave this extension to subsequent studies.
	
\begin{figure}[pt]
	\centering
	\includegraphics[width=0.35\textwidth,height=0.20\textheight]{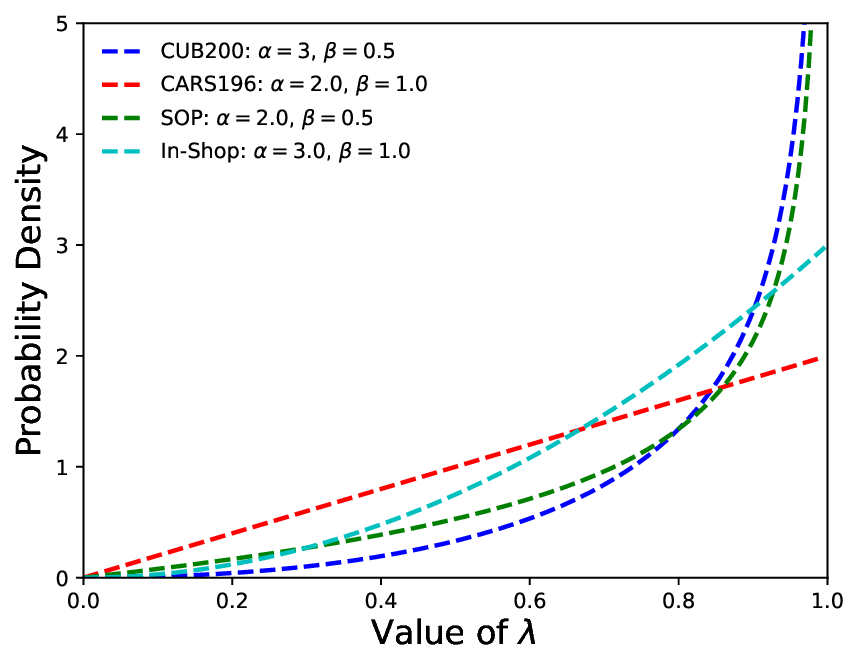}
	\caption{Illustrate the optimized beta distribution on different datasets in our experiments.}
	\label{fig:beta}
	\vspace{-10pt}
\end{figure}


\section{Additional Studies}
\label{appendix:additional}

\subsection{Study of the Mixing Distribution}

We study the mixing domain's performance with various sampling distributions. We evaluate the performance of our method under different probability density functions based on the $\alpha, \beta$. We discover that the optimized sampling density varies in different datasets. This is because of the variety of data distribution in the initial space. When the $\lambda$ sampling approach to $0.0$ is equal to putting the target proxy directly as the target space without any mixing, and vice versa. As illustrated in Figure \ref{fig:beta}, we observe that the $\lambda$ tends to be sampled larger than 0.5 for all datasets but with different probabilities, which means the mixed intermediate data domain tends to have more semantic information from the source data domain in all datasets. This also consists of our assumption that the proxy profoundly needs to be connected with the informative source domain during the DML learning progress.

\begin{figure*}[pt]
	\centering
	
	\begin{subfigure}[t]{0.30\textwidth}
		\includegraphics[width=\textwidth, height=0.16\textheight]{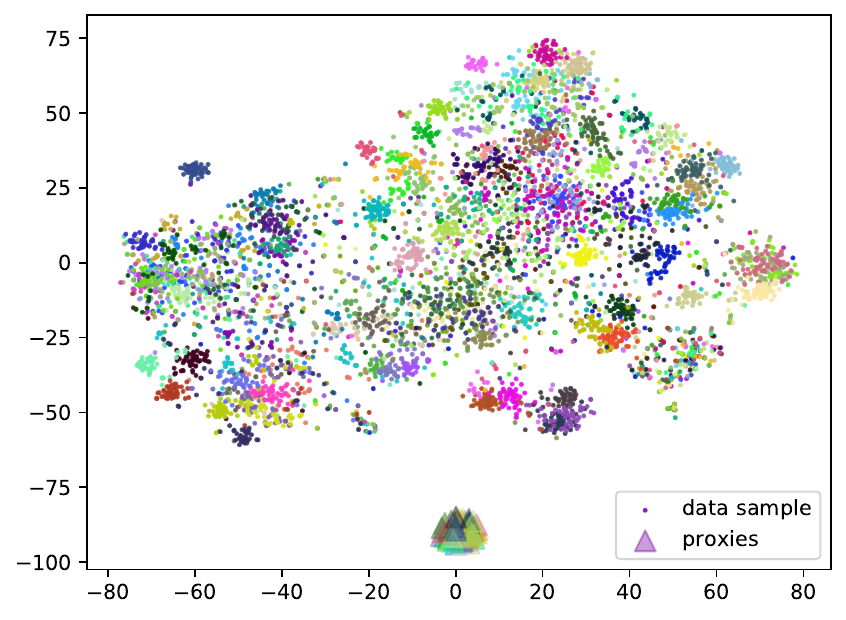}
		\caption{The Initial Space}
		\label{fig:tsne:init}
	\end{subfigure}
	\hspace{-5pt}
	\begin{subfigure}[t]{0.30\textwidth}
		\includegraphics[width=\textwidth, height=0.16\textheight]{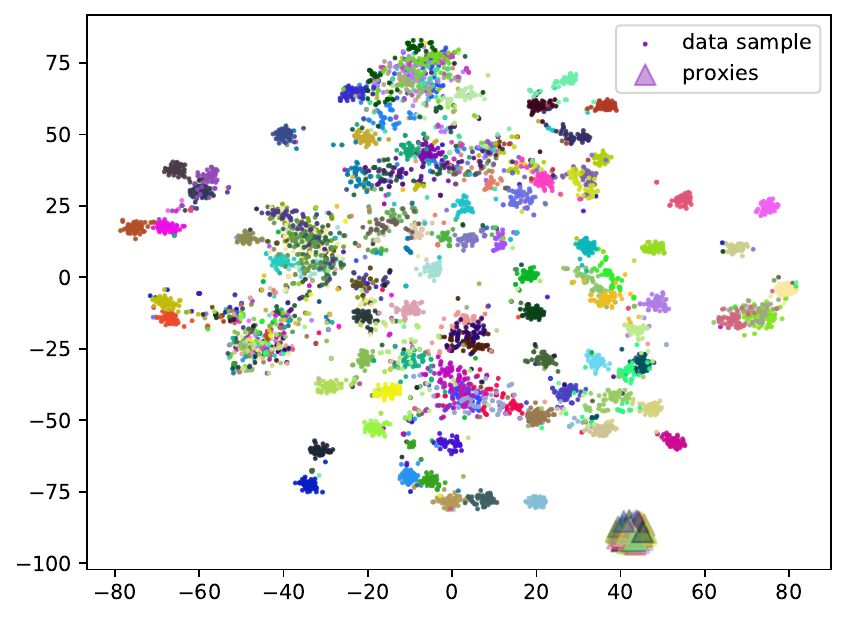}
		\caption{The Space of the Original PA}
		\label{fig:tsne:base}
	\end{subfigure}
	\hspace{-5pt}
	\begin{subfigure}[t]{0.30\textwidth}
		\includegraphics[width=\textwidth, height=0.16\textheight]{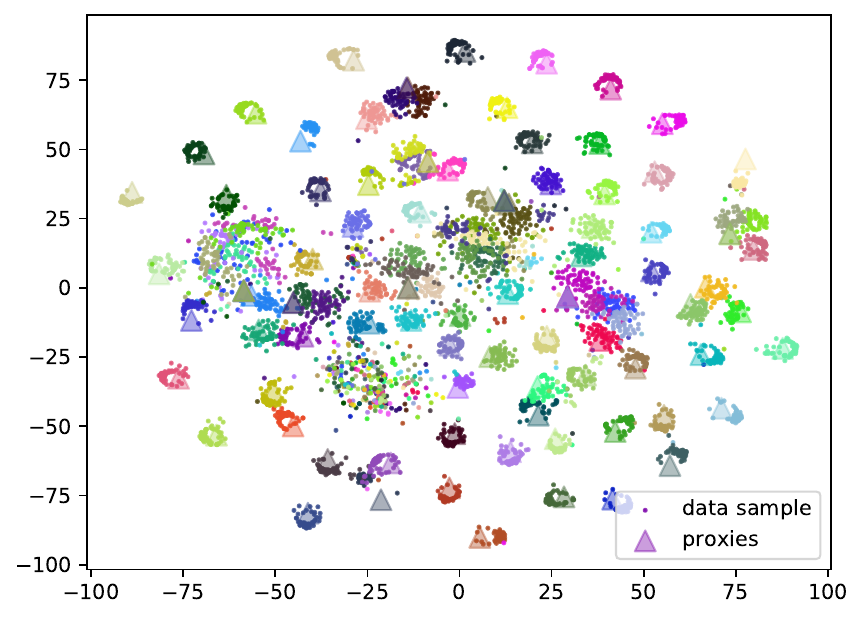}
		\caption{The Space of our PA-DADA}
		\label{fig:tsne:adv}
	\end{subfigure}
	\caption{Illustrate the space with the t-SNE visualizations of sample representations and corresponding proxy vectors on the part of the training set of CUB-200-2011. Each class is labeled with a unique color. We demonstrate that the embedding space of our method looks well clustered, and the proxies are sufficiently separated and close to the clustered sample data, while the proxies of the original PA are not well separated and still maintain their own distribution.}
	\label{fig:tsne}
\end{figure*}

\subsection{Study of Confusion Matrix}
We compare the predicted probability on CUB200 from our category-level discriminator between the data source domain and target proxies. We repeat the comparison on the PA baseline, where we also train a classifier with the same architecture and initial parameters (so that the prediction on the source data is the same as our proposed method). The Figure \ref{fig:mat} illustrates the confusion matrix. It is obvious that our proposed method has better consistency between the distribution of data samples and the proxies.

\subsection{Study of the Adaptation Groups} Another important factor that affects the results is where to apply the domain adaptation. To investigate this we try domain adaptation on different combinations of data in the set $\tilde{X}$, $\tilde{D}$ and $P$. We study the efficiency of our domain-level discriminator under three different settings: (1) the original space with only the augmented samples $\tilde{X}$ and the original proxies $P$ adapted (labeled as $\tilde{X}$ + $P$), (2) the setting with only the augmented samples and $\tilde{X}$ and augmented data $\tilde{D}$ (labeled as $\tilde{X}$ + $\tilde{D}$), and (3) setting with all three set of data as proposed in main paper (labeled as $\tilde{X}$ + $\tilde{D}$ + $P$). For the settings (1) and (2), we set the domain label to 01 and 10 instead. As illustrated in Figure \ref{fig:parameter:r1}, the $\tilde{X}$ + $\tilde{D}$ + $P$ setting achieves the best overall performance and specifically overcomes the $\tilde{X}$ + $P$ setting for a large margin. This demonstrates the efficiency of the augmented data for domain adaptation. 

\section{Detailed Experimental Settings}
\label{appendix:setting}

\subsection{Pre-Processing} 
We follow the standard pre-processing procedure proposed in existing works for fairness. Specifically, we resize the image to $224 \times 224$, do random resized cropping, and random horizontal flipping with probability 0.5. In the test phase, the images are first resized to $256 \times 256$, then cropped back to $224 \times 224$.


\begin{figure*}[t]
	\centering
	\begin{subfigure}[t]{0.24\textwidth}
		\includegraphics[width=\textwidth, height=0.15\textheight]{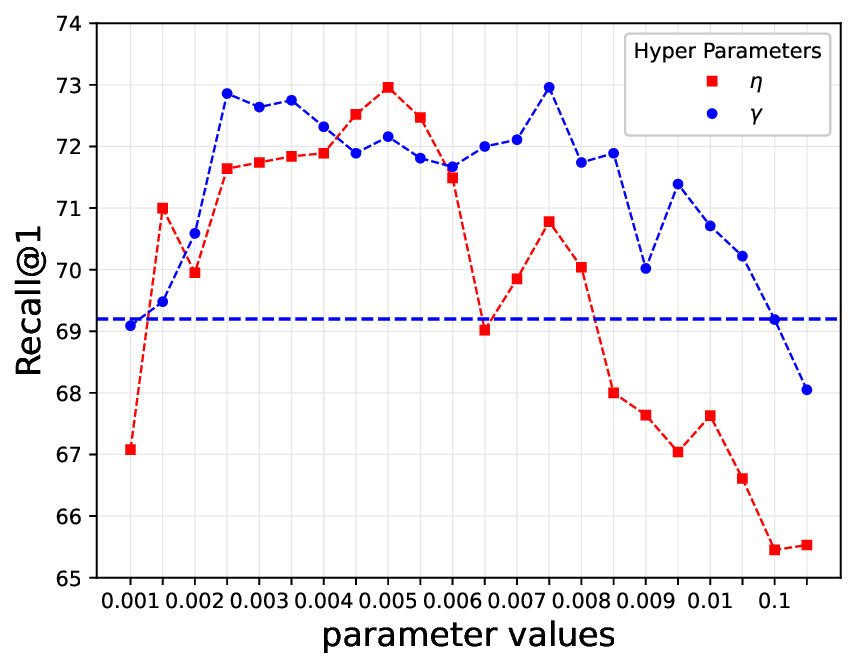}
		\caption{Impact of $\eta$ and $\gamma$}
		\label{fig:parameter:par}
	\end{subfigure}
	\hspace{-5pt}
	\begin{subfigure}[t]{0.24\textwidth}
		\includegraphics[width=\textwidth, height=0.15\textheight]{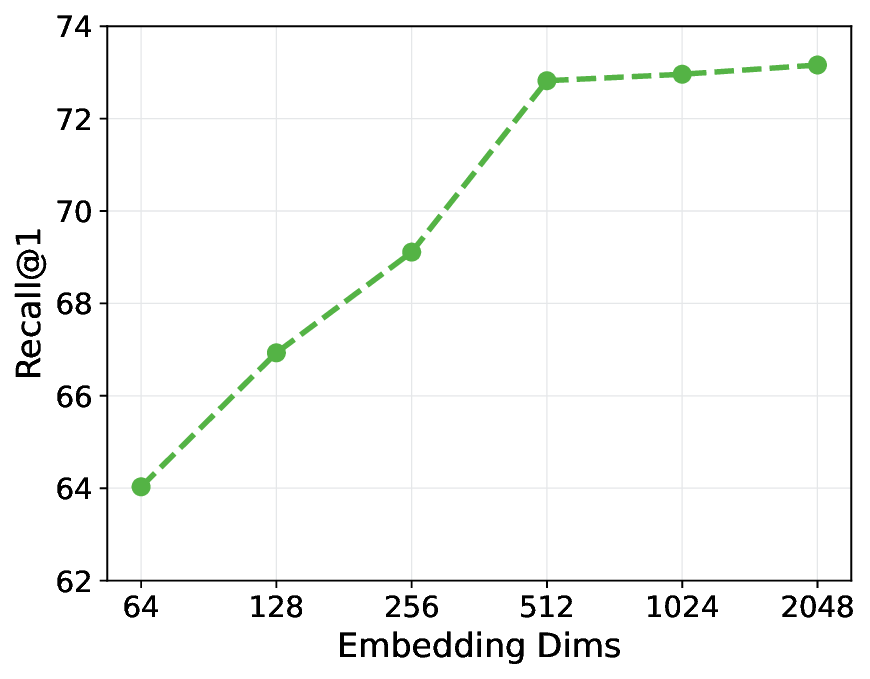}
		\caption{Impact of dimension}
		\label{fig:parameter:dim}		
	\end{subfigure}
	\hspace{-5pt}
	\begin{subfigure}[t]{0.24\textwidth}
		\includegraphics[width=\textwidth, height=0.15\textheight]{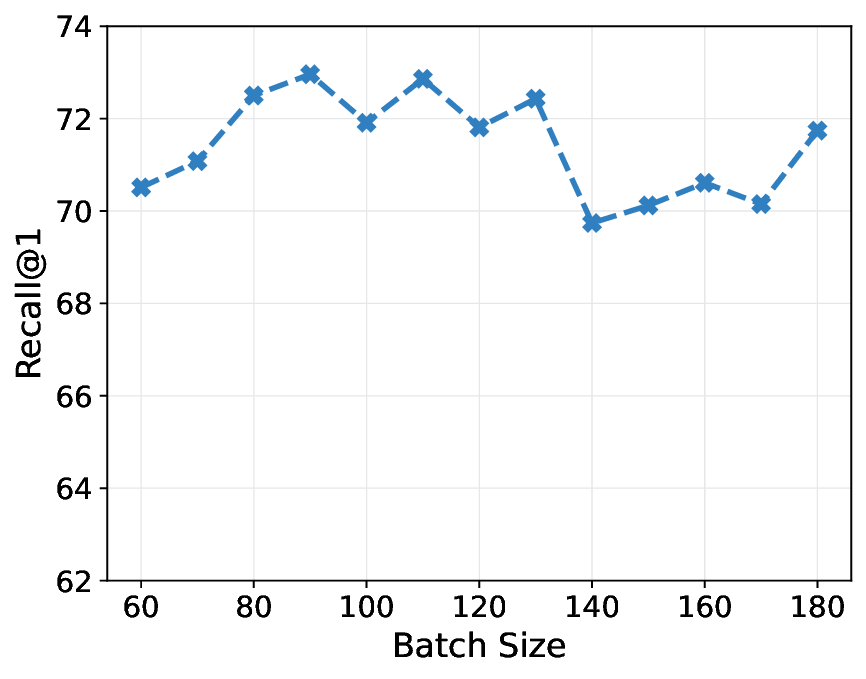}
		\caption{Impact of batch size}
		\label{fig:parameter:batch}
	\end{subfigure}
	\hspace{-5pt}
	\begin{subfigure}[t]{0.24\textwidth}
		\includegraphics[width=\textwidth, height=0.15\textheight]{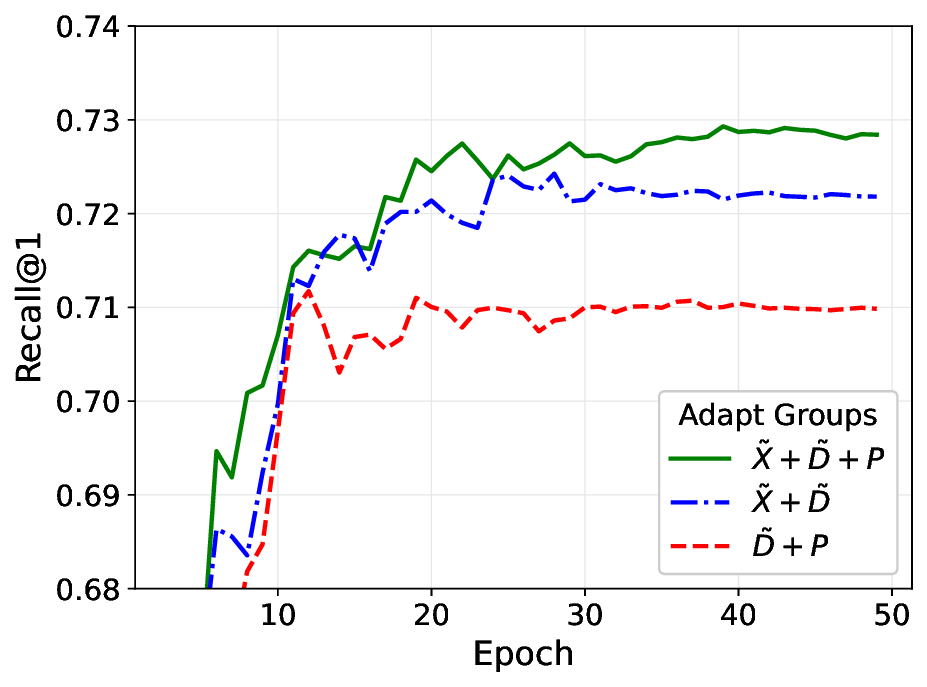}
		\caption{Compare Adapted Groups}
		\label{fig:parameter:r1}
	\end{subfigure}
	
	\caption{Illustrate the impact of our hyperparameters. The blue line represents our Proxy-Anchor baseline. The blue horizontal line indicates our PA baseline.}
	\label{fig:parameter}
	
\end{figure*}

\subsection{Parameter Searching}
We empirically search the dimensions of each hidden layer of our discriminators in a discrete range of $\{128, 256, 512\}$. We search the second layer of the category-level classifier according to the number of classes in each dataset. We eventually set our domain-level discriminator as a single-layer MLP with a 512-dimension hidden layer. Our category-level discriminator is constructed as a 2-layer MLP where the first hidden layer has 512 dimensions.

We search our hyper-parameters $\{\eta, \gamma\}$ from a grid where $\eta, \gamma$ are searched in a range from $10^{-2}$ to $1.0$. For the parameter $\alpha, \beta$ in Beta distribution, we also search them within the range $\alpha, \beta \in [0.5, 5]$ since they control the probability of sampling between the two domains. We adopt the parameters of baseline Proxy-NCA and Proxy-Anchor as their suggested values in the original paper where the scaling factor $\tau=32$, margin $\delta=0.1$. In each iteration of training, we train the discriminator with $k=3$ times while training the generator for a single time. We adopt a warm-up phase where all parameters are frozen except the linear head of the generator, as suggested in \cite{roth2022non}.

\subsection{Impact of Hyperparameters} We observe in Figure \ref{fig:parameter} that our discriminators are effective when the balancing factor satisfies $ 0.002 \leq \eta \leq 0.006$ and $0.002 \leq \gamma \leq 0.009$. In this range, the performance starts to overcome the settings with a single classifier (R@1 $>70.0$). Besides, we notice that the adversarial training would not support or even damage the original metric learning system when $\eta \geq 0.08, \gamma \geq 0.1$.

We also observe that the large batch size is not a sufficient condition to achieve better performance for our method (as illustrated in Figure \ref{fig:parameter:batch}) since the batch size and batch sampling also affect the quality of our mixed positive samples. Figure \ref{fig:parameter:dim} also illustrates that the performance of our learning system improves with the dimension size of the features. Even though we achieve better results on larger dimensions, we only report the results on dimension 512 for fairness with other works. Note that our method is based on Proxy-Anchor where the proxy collects information across batches. Thus, the performance of our method is not sensitive to batch size (see Figure \ref{fig:parameter:batch}).

\subsection{Training Time and Memory} 
When utilizing our adversarial training, the additional cost of time and memory is very limited in our framework. In our experimental machine, the training time on CUB200 for a single epoch of our proposed method is $37.3\pm0.4s$ (average in 10 epochs), including the evaluation time, when the original PA baseline holds the running time of $35.2\pm0.3s$. Even though we train the discriminators in $k=3$ times in each iteration, the increased time for this additional part is only around $2s(5\%)$. This is because our discriminators are shallow (1 and 2 layers), and the computational complexity does not change. In terms of memory, our proposed method takes 11.3GB of GPU memory when the original baseline takes 11.2GB for the default batch size of 90. Although we increase the number of samples by mixing the features, the increased space is negligible compared to the baseline since our proposed method does not need additional forward propagation for the increased samples.

\section{Broader Impact and Limitation}
\label{appendix:limits}
We believe our DADA work would serve as an excellent example to inspire other researchers on proxy-based DML to focus on capturing space distributions. By investigating more solutions to compare and align the domains, future contributions would further push their limits under an aligned space with unique data distribution. The possible limitation of DADA is its learning efficiency on datasets with a large number of classes, which may cause difficulty in extending to other learning paradigms, such as continuous learning. We will focus on this interesting challenge in future works.

	\begin{figure*}[pt]
		\centering
		\begin{subfigure}[t]{0.30\textwidth}
			\includegraphics[width=\textwidth, height=0.16\textheight]{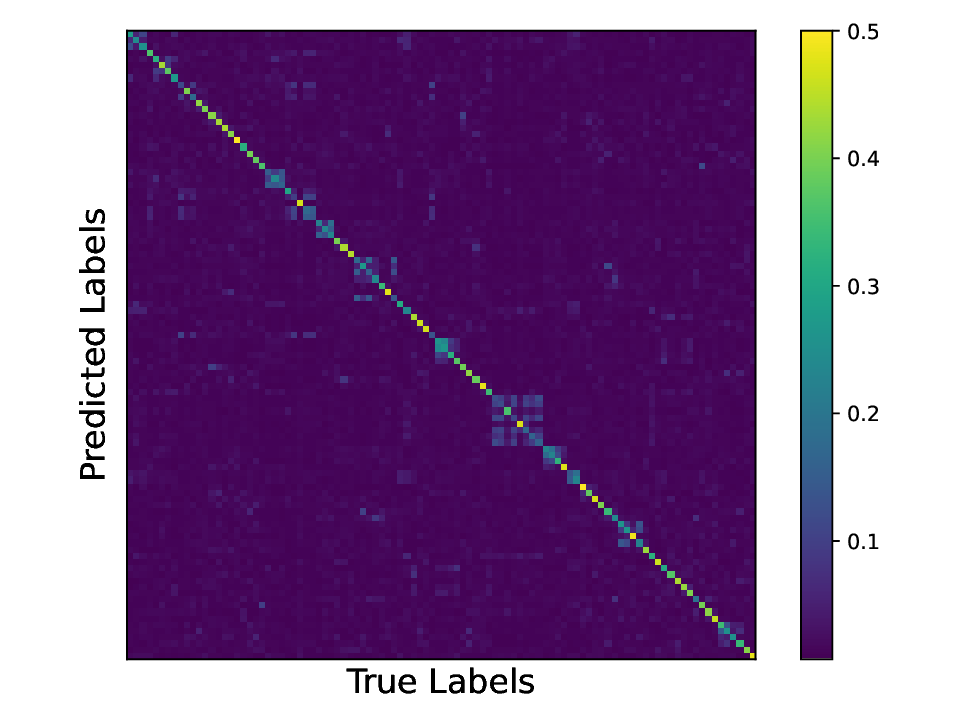}
			\caption{Source data sample}
			\label{fig:mat:source}
		\end{subfigure}
		\hspace{-5pt}
		\begin{subfigure}[t]{0.30\textwidth}
			\includegraphics[width=\textwidth, height=0.16\textheight]{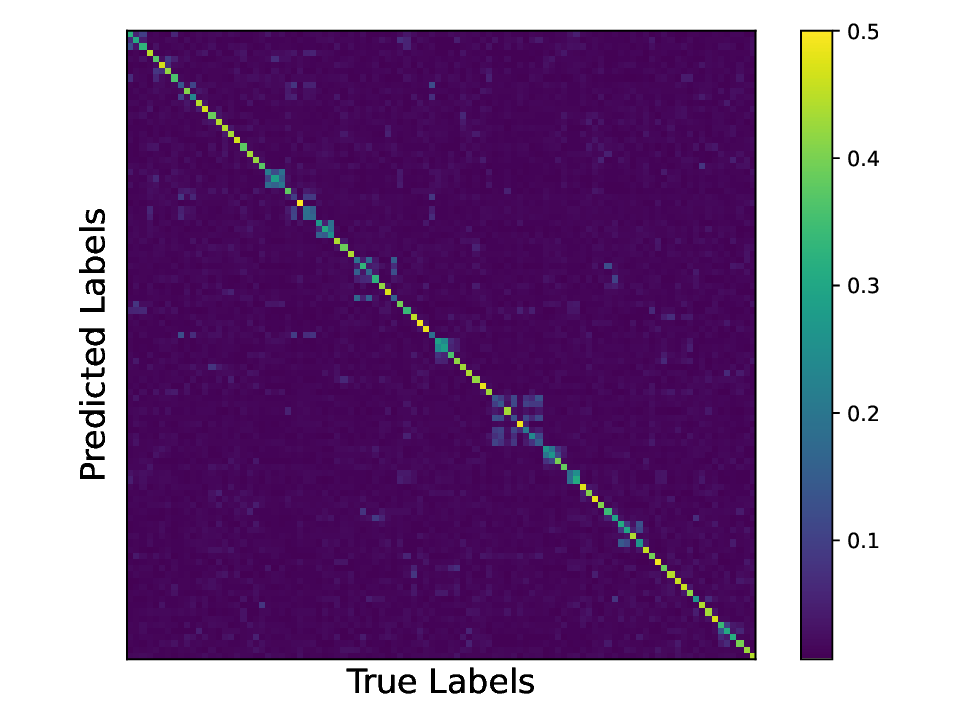}
			\caption{Target proxies of our proposed method}
			\label{fig:mat:adv}
		\end{subfigure}
		\hspace{-5pt}
		\begin{subfigure}[t]{0.30\textwidth}
			\includegraphics[width=\textwidth, height=0.16\textheight]{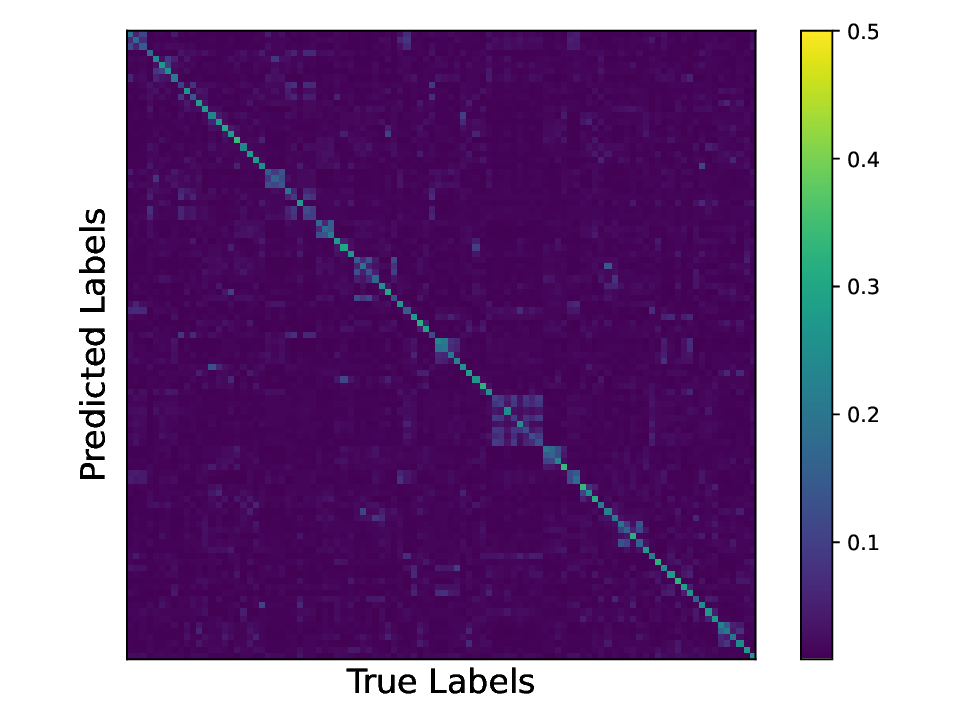}
			\caption{Target proxies of PA baseline}
			\label{fig:mat:base}
		\end{subfigure}
		\caption{Illustrate the confusion matrix on CUB200 (100 classes) from our category-level discriminator. We also illustrate the confusion matrix of the PA baseline where we also train a classifier with the same architecture and initialed parameters in baseline training.}
		\label{fig:mat}
	\end{figure*}

\begin{figure}[pt]
	\vspace{-5pt}
	\centering
	\includegraphics[width=0.45\textwidth,height=0.25\textheight]{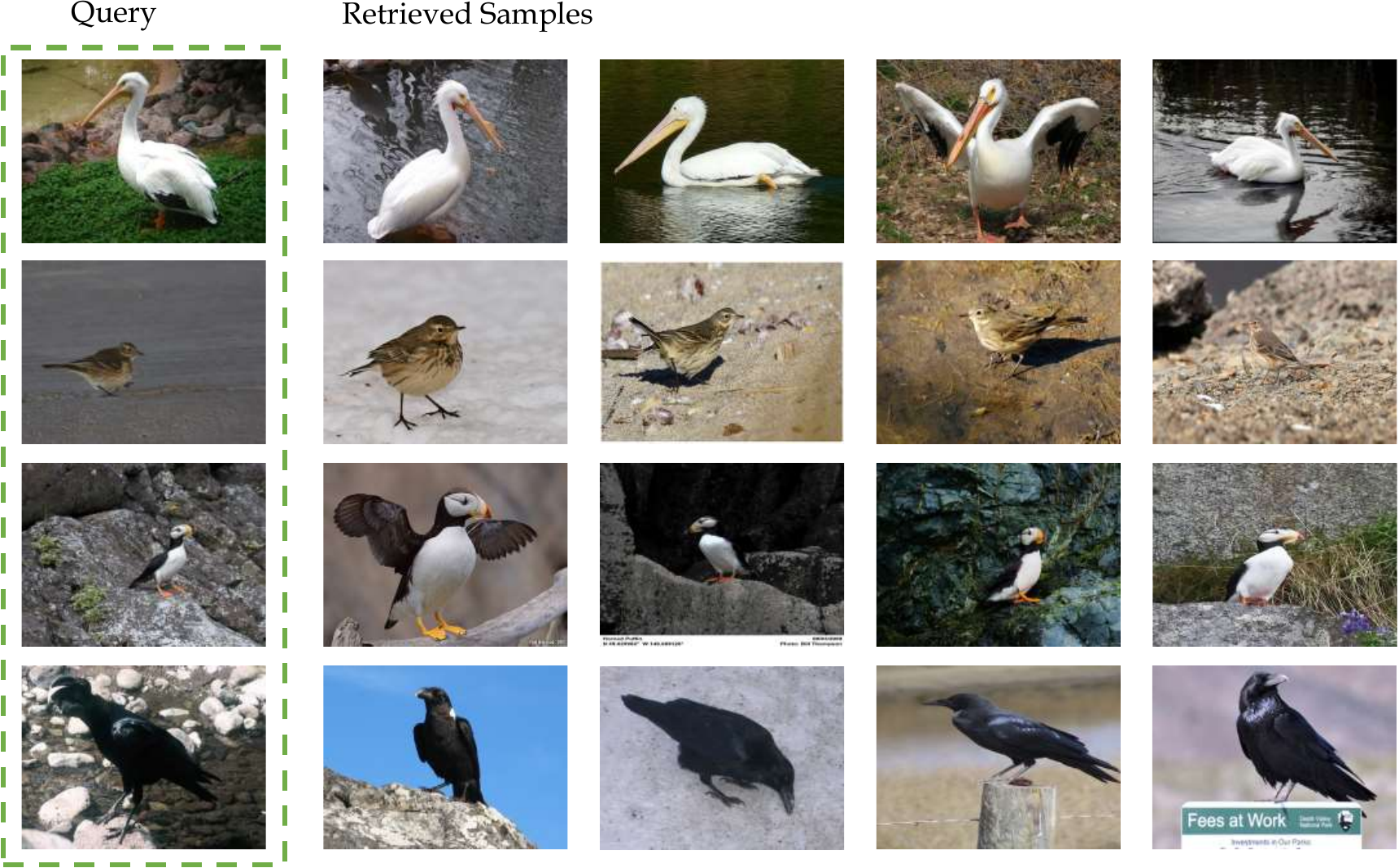}
	\caption{Illustrate the retrieved examples on CUB200 test set. Ranked by similarity score.}
	\label{fig:example_1}
	\vspace{-10pt}
\end{figure}

\begin{figure}[pt]
	\vspace{-5pt}
	\centering
	\includegraphics[width=0.45\textwidth,height=0.25\textheight]{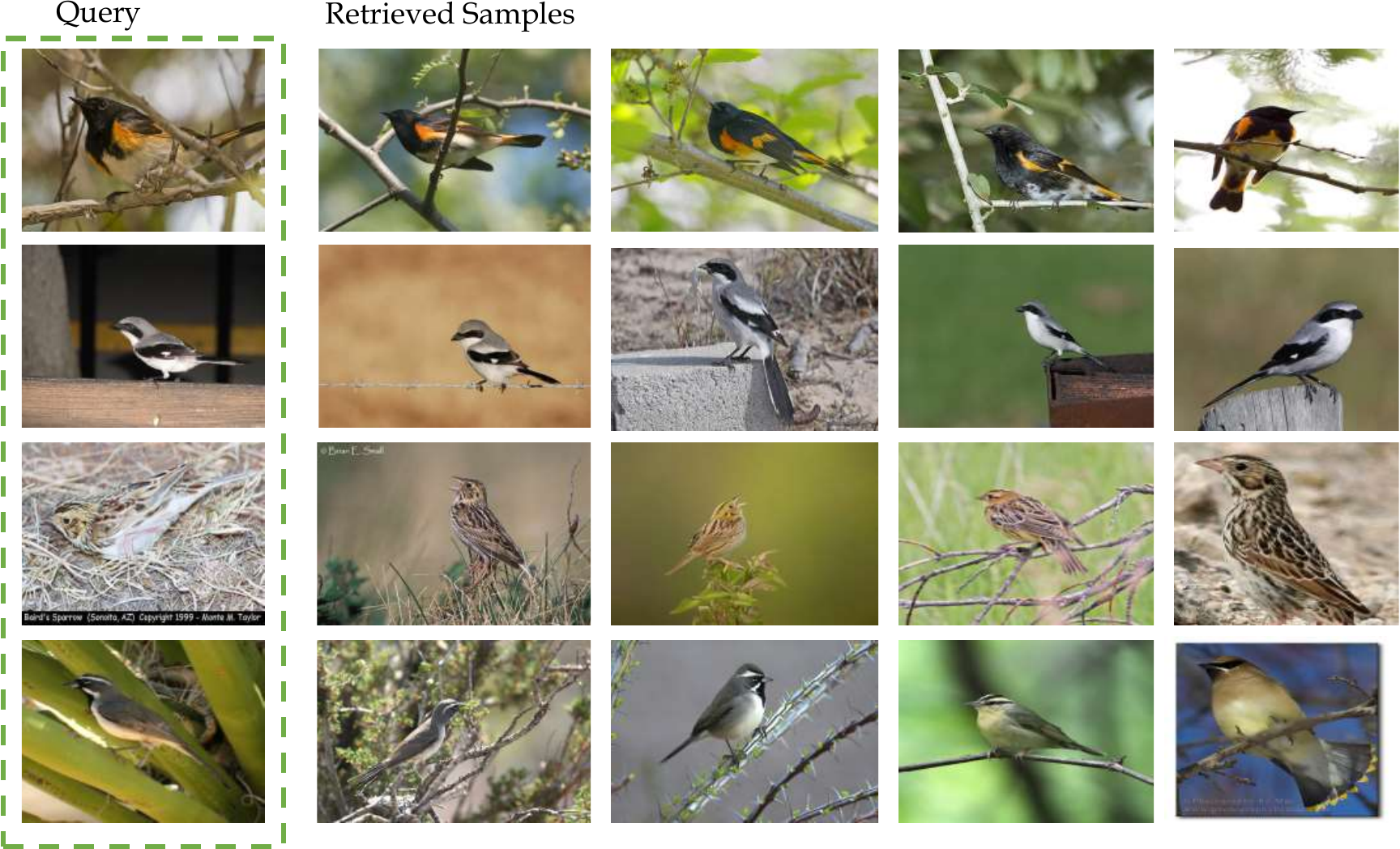}
	\caption{Illustrate the retrieved examples on CUB200 test set. Ranked by similarity score.}
	\label{fig:example_2}
	\vspace{-10pt}
\end{figure}

\begin{figure}[pt]
	\vspace{-5pt}
	\centering
	\includegraphics[width=0.45\textwidth,height=0.25\textheight]{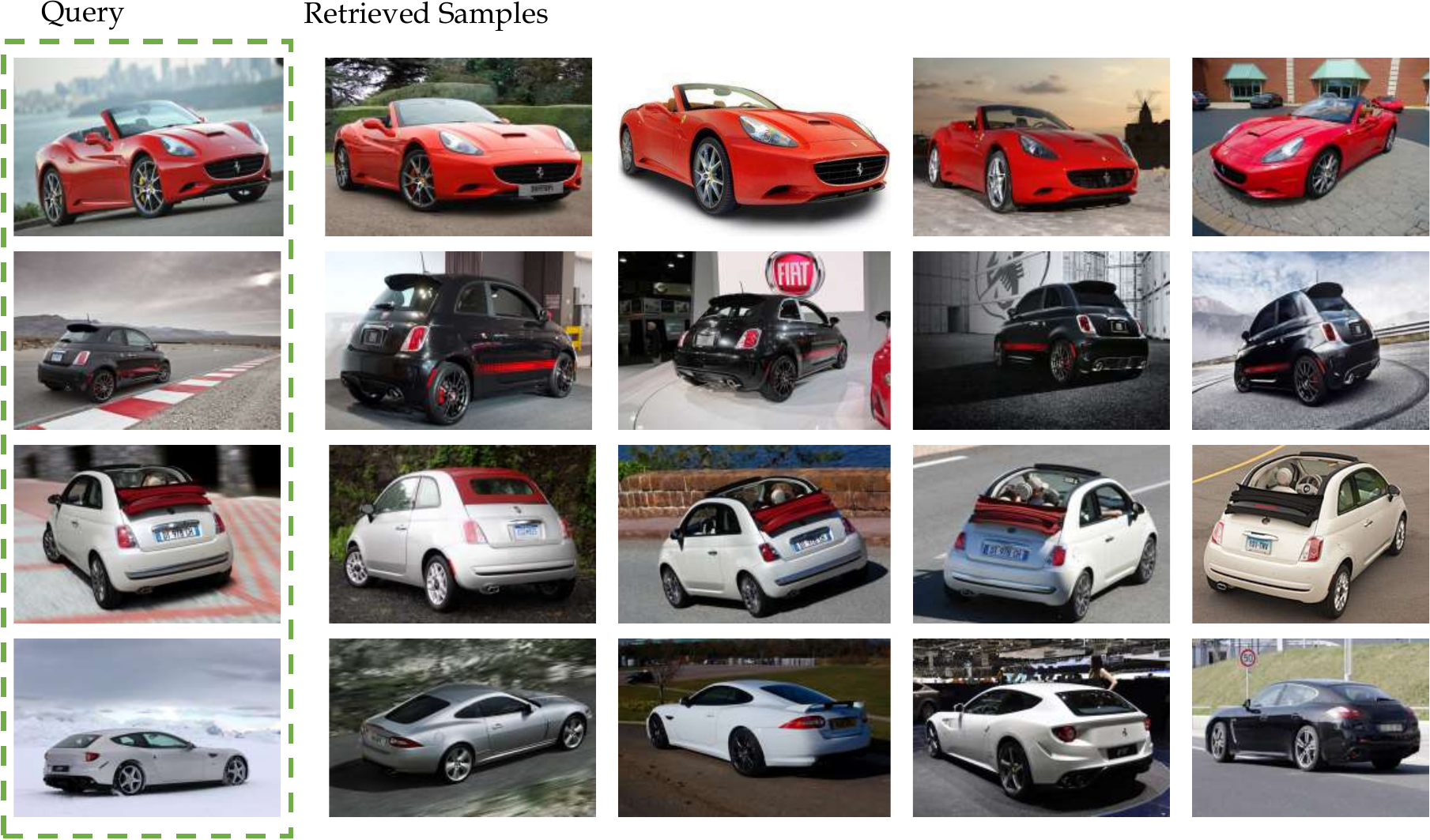}
	\caption{Illustrate the retrieved examples on CARS196 test set. Ranked by similarity score.}
	\label{fig:example_3}
	\vspace{-10pt}
\end{figure}

\begin{figure}[pt]
	\vspace{-5pt}
	\centering
	\includegraphics[width=0.45\textwidth,height=0.25\textheight]{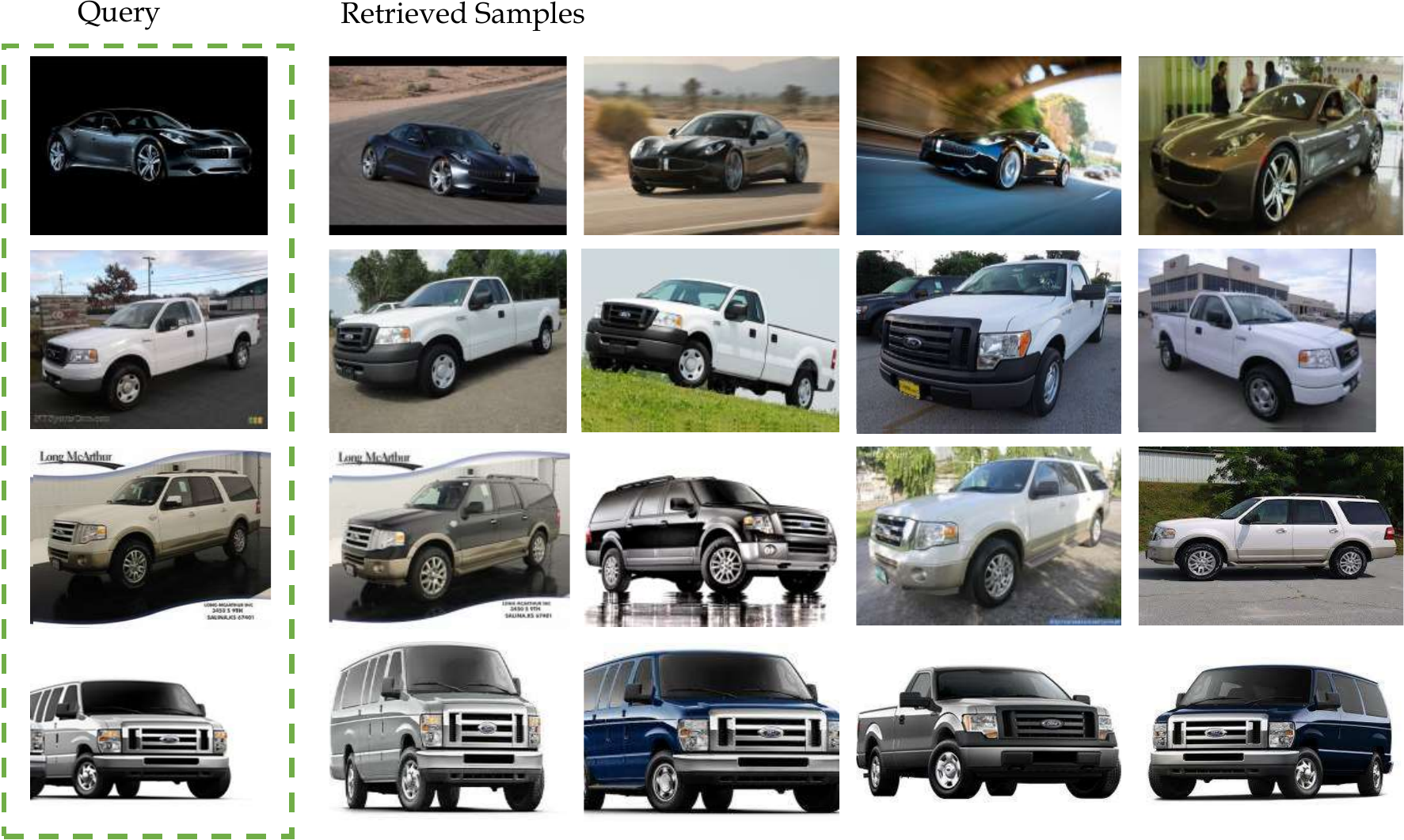}
	\caption{Illustrate the retrieved examples on CARS196 test set. Ranked by similarity score.}
	\label{fig:example_4}
	\vspace{-10pt}
\end{figure}

\end{document}